%% file: neurips_2021.tex
\documentclass{article}

\PassOptionsToPackage{numbers,sort&compress}{natbib}
\usepackage[final]{neurips_2021}

\usepackage[utf8]{inputenc} %
\usepackage[T1]{fontenc}    %
\usepackage{url}            %
\usepackage{booktabs}       %
\usepackage{amsfonts}       %
\usepackage{nicefrac}       %
\usepackage{microtype}      %
\usepackage{xcolor}         %
\usepackage{caption}
\usepackage{subcaption}
\usepackage{wrapfig}
\usepackage{multirow}
\usepackage{numprint}
\npthousandsep{\,}
\usepackage[percent]{overpic}

\input{al_header.tex}

\usepackage{hyperref}
\hypersetup{
       colorlinks,
       linkcolor=ourdarkblue,
       urlcolor=ourdarkblue,
       citecolor=ourdarkred}

\AtBeginDocument{\let\latexlabel\label}

\theoremstyle{definition} \newtheorem{definition}{Definition}
\theoremstyle{plain} \newtheorem{proposition}{Proposition}

\DeclareMathOperator*{\rank}{rank}
\DeclareMathOperator{\Tr}{Tr}
\newcommand{\colonequal}{\mathbin{\vcentcolon=}}
\newcommand{\Pa}{\ensuremath{\mathrm{Pa}}}
\newcommand{\pa}{\ensuremath{\mathrm{pa}}}
\newcommand{\Do}[2]{\ensuremath{{\mathrm{do}\g(#1 \colonequal #2\g)}}}
\newcommand{\CAI}{\ensuremath{C^j}\xspace}
\newcommand{\CAIP}{CAI-P\xspace}

\newcommand{\FPP}{\textsc{FetchPickAndPlace}\xspace}%
\newcommand{\FP}{\textsc{FetchPush}\xspace}
\newcommand{\FRT}{\textsc{Fetch\allowbreak RotTable}\xspace} %
\newcommand{\Slide}{\textsc{{1\kern-1pt DSlide}}\xspace}

\title{Causal Influence Detection \\for Improving Efficiency in Reinforcement Learning}

\author{%
  Maximilian Seitzer \\
  MPI for Intelligent Systems \\
  Tübingen, Germany \\
  \hspace{-1.5em}\texttt{maximilian.seitzer@tue.mpg.de}\hspace{-3em} \\
  \And
  \hspace{1.5em}Bernhard Sch\"olkopf \\
  \hspace{1.5em}MPI for Intelligent Systems \\
  \hspace{1.5em}Tübingen, Germany \\
  \hspace{1.7em}\texttt{bs@tue.mpg.de} \\
  \And
  \hspace{-0.3em}Georg Martius \\
  \hspace{-0.3em}MPI for Intelligent Systems \\
  \hspace{-0.3em}Tübingen, Germany \\
  \hspace{-0.35em}\texttt{georg.martius@tue.mpg.de} \\
}

\begin{document}

\maketitle
\begin{abstract}
Many reinforcement learning (RL) environments consist of independent entities that interact sparsely.
In such environments, RL agents have only limited influence over other entities in any particular situation.
Our idea in this work is that learning can be efficiently guided by knowing when and what the agent can influence with its actions.
To achieve this, we introduce a measure of \emph{situation-dependent causal influence} based on conditional mutual information and show that it can reliably detect states of influence.
We then propose several ways to integrate this measure into RL algorithms to improve exploration and off-policy learning.
All modified algorithms show strong increases in data efficiency on robotic manipulation tasks.
\end{abstract}

\section{Introduction}
\label{sec:introduction}
Reinforcement learning (RL) is a promising route towards versatile and dexterous artificial agents.
Learning from interactions can lead to robust control strategies that can cope with all the intricacies of the real world that are hard to engineer correctly.
Still, many relevant tasks such as object manipulation pose significant challenges for RL.
Although impressive results have been achieved using simulation-to-real transfer~\citep{andrychowicz2020:inhandlearning} or heavy physical parallelization~\citep{Kalashnikovetal2018:QTOpt}, training requires countless hours of interaction. 
Improving sample efficiency is thus a key concern in RL.
In this paper, we approach this issue from a causal inference perspective.

When is an agent in control of its environment?
An agent can only influence the environment by its actions.
This seemingly trivial observation has the underappreciated aspect that the causal influence of actions is \emph{situation dependent}.
Consider the simple scenario of a robotic arm in front of an object on a table.
Clearly, the object can only be moved when contact between the robot and object is made.
Generally, there are situations where immediate causal influence is possible, while in others, none is.
In this work, we formalize this situation-dependent nature of control and show how it can be exploited to improve the sample efficiency of RL agents.
To this end, we derive a measure that captures the causal influence of actions on the environment and devise a practical method to compute it.

Knowing when the agent has control over an object of interest is important both from a learning and an exploration perspective. 
The learning algorithm should pay particular attention to these situations because (i) the robot is initially rarely in control of the object of interest, making training inefficient, (ii) physical contacts are hard to model, thus require more effort to learn and (iii) these states are enabling manipulation towards further goals.
But for learning to take place, the algorithm first needs data that contains these relevant states.
Thus, the agent has to take its causal influence into account already during exploration.

We propose several ways in which our measure of causal influence can be integrated into RL algorithms to address both the exploration, and the learning side.
For exploration, agents can be rewarded with a bonus for visiting states of causal influence.
We show that such a bonus leads the agent to quickly discover useful behavior even in the absence of task-specific rewards.
Moreover, our approach allows to explicitly guide the exploration to favor actions with higher predicted causal impact.
This works well as an alternative to $\epsilon$-greedy exploration, as we demonstrate.
Finally, for learning, we propose an off-policy prioritization scheme and show that it reliably improves data efficiency.
Each of our investigations is backed by empirical evaluations in robotic manipulation environments and demonstrates a clear improvement of the state-of-the-art with the same generic influence measure.

\section{Related Work}
\label{sec:related_work}

The idea underlying our work is that an agent can only sometimes influence its surroundings.
This rests on two basic assumptions about the causal structure of the world.
The first is that the world consists of independent entities, in accordance with the principle of independent causal mechanisms (ICM)~\cite{Peters2017CausalInference}, stating that the world's generative process consists of autonomous modules.
The second assumption is that the potential influence that entities have over other entities is localized spatially and occurs sparsely in time.
We can see this as explaining the sparse mechanism shift hypothesis, which states that naturally occurring distribution shifts will be due to local mechanism changes~\cite{Schoelkopf2021CausalReprLearning}.
This is usually traced back to the ICM principle, \ie that interventions on one mechanism will not affect other mechanisms~\cite{Parascandolo2018ICM}.
But we argue that it is also due to the \emph{limited interventional range} of agents (or, more generally, physical processes), which restricts the breadth and frequency of mechanism-changes in the real world.
Previous work has used sparseness to learn disentangled representations~\cite{Bengio2020MetaTransfer,Locatello2020WeaklySupDis}, causal models~\cite{Ke2019NeuralCausalModels}, or modular architectures~\cite{Goyal2021RIM}.
In the present work, we show that taking the localized and sparse nature of influence into account can also strongly improve RL algorithms.

Detecting causal influence, informally, means deciding whether changing a causal variable would have an impact on another variable.
This involves causal discovery, that is, finding the existence of arrows in the causal graph~\cite{Pearl2009Causality}.
While the task of causal discovery is unidentifiable in general \cite{Spirtes2000}, there are assumptions which permit discovery \cite{PetMooJanSch11}, in particular in the time series setting we are concerned with \cite{Eichler2012CausalInf}.
Even if the existence of an arrow is established, the problem remains of quantifying its causal impact, for which various measures such as transfer entropy or information flow have been proposed~\cite{Schreiber2000TransferEnt,Ay2007GeometricRobustness,Ay2008InformationFlow,Lizier2013InfoDynamics,Janzing2013CausalInf}.
We compare how our work relates to these measures in \sec{subsec:measuring_cai}.

The intersection of RL and causality has been the subject of recent research~\cite{Bareinboim2015,Lu2018DeconfoundingRL, Rezende2020CausalPartialModels,Buesing2019WouldaCS,Bareinboim2019NearOptimalRL}.
Close to ours is the work of~\citet{Pitis2020CFDataAug}, who also use influence detection, albeit to create counterfactual data that augments the training of RL agents.
In \sec{sec:empirical_eval}, we find that our approach to action influence detection performs better than their heuristic approach.
Additionally, we demonstrate that influence detection can also be used to help agents explore better.
To this end, we use influence as a type of intrinsic motivation. 
For exploration, various signals have been proposed, \eg model surprise~\cite{Schmidhuber1991Curious,Pathak2017ICM,BlaesVlastelicaZhuMartius2019:CWYC}, learning progress~\cite{Colas2018:CURIOUS,BlaesVlastelicaZhuMartius2019:CWYC}, empowerment~\cite{Klyubin2005Empowerment,Mohamed2015Empowerment}, information gain~\cite{Storck95:RL-infogain, LittleSommer2013:PIG,Houthooft2016VIME}, or predictive information~\cite{MartiusDerAy2013, ZahediMartiusAy2013}.
Inspired by causality, \citet{Sontakke2020CausalCuriosity} introduce an exploration signal that leads agents to experiment with the environment to discover causal factors of variation.
In concurrent work, \citet{Zhao2021MUSIC} propose to use mutual information between the agent and the environment state for exploration.
As in our work, the agent is considered a separate entity from the environment.
However, their approach does not discriminate between individual situations the agent is in.
Causal influence is also related to the concept of contingency awareness from psychology~\cite{Watson1966Contingency}, that is, the knowledge that one's actions can affect the environment.
On Atari games, exploring through the lens of contingency awareness has led to state-of-the-art results~\cite{Choi2018ContingencyAware,Song2020Mega}.

\section{Background}
\label{sec:background}

We are concerned with a Markov decision process $\langle \mathcal{S}, \mathcal{A}, P, r, \gamma \rangle$ consisting of state and action space, transition distribution, reward function and discount factor.\footnote{We use capital letters (\eg $X$) to denote random variables, small letters to denote samples drawn from particular distributions (\eg $x \sim P_X$), and caligraphy letters to denote graphs, sets and sample spaces (\eg $x \in \mathcal{X}$). We denote distributions with $P$ and their densitites with $p$.}
Most real world environments consist of entities that behave mostly independently of each other.
We model this by assuming a known state space factorization $\mathcal{S} = \mathcal{S}_1 \times  \ldots \times \mathcal{S}_N$, where each $\mathcal{S}_i$ corresponds to the state of an entity.

\begin{figure}\centering
    \begin{subfigure}[b]{0.39\textwidth}
        \includegraphics[width=\textwidth]{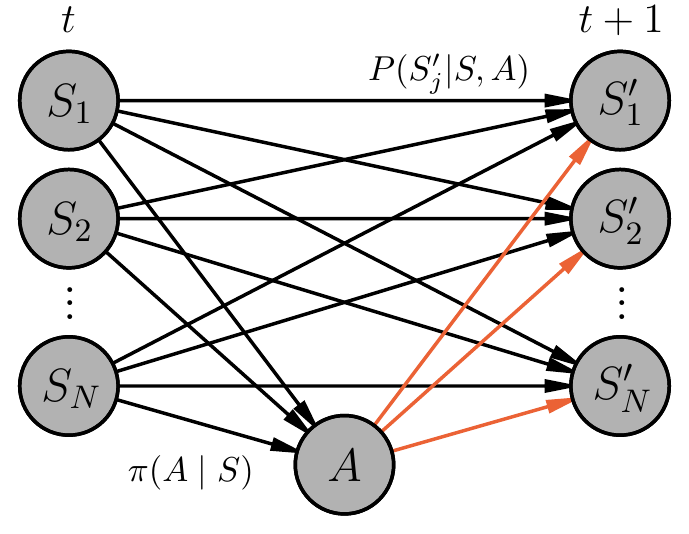}
        \caption{Causal Graph $\mathcal{G}$}
        \label{subfig:causal_graph}
    \end{subfigure}
    \quad
    \begin{minipage}[b]{0.45\textwidth}
    \begin{subfigure}[b]{\textwidth}
        \begin{center}
        \includegraphics[width=0.47\textwidth]{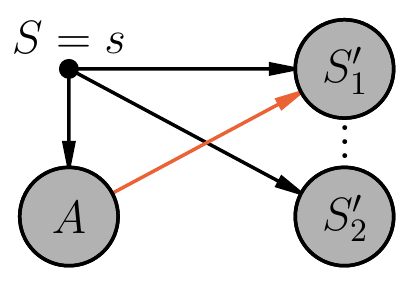}%
        \hfill
        \raisebox{.2\height}{\includegraphics[width=0.45\textwidth]{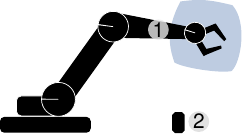}}
        \vspace*{-0.5em}
        \caption[skip=0]{No influence of $A$ on $S_2'$}
        \label{subfig:causal_graph_no_influence}
        \end{center}
    \end{subfigure}
    \begin{subfigure}[b]{\textwidth}
        \begin{center}
        \includegraphics[width=0.47\textwidth]{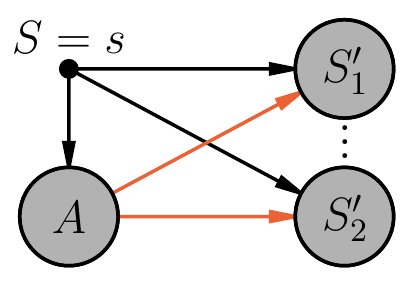}%
        \hfill
        \raisebox{.2\height}{\includegraphics[width=0.45\textwidth]{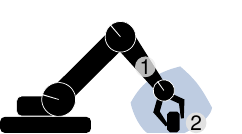}}
        \vspace*{-0.5em}
        \caption{Influence of $A$ on $S_1'$ and $S_2'$}
        \label{subfig:causal_graph_influence}
        \end{center}
    \end{subfigure}
    \end{minipage}
    \vspace{-.3em}
    \caption{Causal graphical model capturing the environment transition from state $S$ to $S'$ by action $A$, factorized into state components.
    (a): Viewed globally over all time steps, all components of the state and the action can influence all state components at the next time step.
    (b, c): Given a situation $S=s$, some influences may or may not not hold in the \emph{local causal graph} $\mathcal{G}_{S=s}$.
    In this paper, our aim is to detect which influence the action has on $S'$, \ie the presence of the red arrows.}
    \label{fig:causal_graph}
\end{figure}

\subsection{Causal Graphical Models}
\label{subsec:cgm}

We can model the one-step transition dynamics at time step $t$ using a \emph{causal graphical model} (CGM)~\cite{Pearl2009Causality,Peters2017CausalInference} over the set of random variables $\mathcal V=\{S_1, \ldots, S_N, A, S'_1, \ldots, S'_N\}$, consisting of a directed graph $\mathcal{G}$ (see \fig{subfig:causal_graph}) and a conditional distribution $P(V_i \mid \Pa(V_i))$ for each node $V_i \in \mathcal V$, where $\Pa(V_i)$ is the set of parents of $V_i$ in the causal graph.
We assume that the joint distribution $P_\mathcal{V}$ is Markovian with respect to the graph~\cite[Def. 6.21 (iii)]{Peters2017CausalInference}, that is, its density exists and factorizes as
\begin{align}\label{eq:causal_markov_condition}
p\g(v_1, \ldots, v_{|\mathcal V|}\g) = \prod_{i=1}^{|\mathcal V|} p\g(v_i \mid \Pa(V_i)\g).
\end{align}
In a CGM, we can model a (stochastic) intervention $\Do{V_i}{q\g(v_i \mid \text{Pa}(V_i)\g)}$ on variable $V_i$ by replacing its conditional $p\g(v_i \mid \Pa(V_i)\g)$ in \eqn{eq:causal_markov_condition} with the distribution $q\g(v_i \mid \Pa(V_i)\g)$~\cite{Peters2017CausalInference}.
Here, $V_i$ could \eg be a state component $S_i$, or the agent's action $A$.
Thus, whereas a probabilistic graphical model represents a single distribution, a CGM represents a set of distributions~\cite{Schoelkopf2021CausalReprLearning}.

The causal graph that we assume is shown in \fig{subfig:causal_graph}.
Within a time step, there are no edges, \ie no instantaneous effects, except for the action which is computed by the policy $\pi(A \mid S)$.
Between time steps, the graph is fully connected.
The reason is that whenever an interaction between two components $S_i$ and $S_j$, however unlikely, is possible, it is necessary to include an arrow $S_i \rightarrow S_j'$ (and vice versa).
Nevertheless, during most concrete time steps, there should be \emph{no} interaction between entities, reflecting the assumption that the state components represent independent entities in the world.
In particular, the agent's ``sphere of influence'' (depicted in blue in Figs.~\ref{subfig:causal_graph_no_influence} and \ref{subfig:causal_graph_influence}) is limited -- its action $A$ can only sparsely affect other entities.
Thus, in this paper, we are interested in inferring the influence the action has in a specific state configuration $S=s$, that is, the \emph{local causal model} in $s$. %

\begin{definition}\label{def:local_cgm}
 Given a CGM with distribution $P_\mathcal{V}$ and graph $\mathcal{G}$, we define the \emph{local CGM} induced by observing $X=x$ with $X \subset \mathcal{V}$ to be the CGM with joint distribution $P_{\mathcal{V} \mid X=x}$ and the graph $\mathcal{G}_{X=x}$ resulting from removing edges from $\mathcal{G}$ until $P_{\mathcal{V} \mid X=x}$ is causally minimal with respect to the graph.
\end{definition}
Causal minimality tells us that each edge $X \rightarrow Y$ in the graph must be ``active'', in the sense that $Y$ is conditionally dependent on $X$ given all other parents of $Y$~\cite[Prop.~6.36]{Peters2017CausalInference}.

\subsection{The Cause of an Effect}

When is an agent's action $A=a$ the cause of an outcome $B=b$?
Answering this question precisely is surprisingly non-trivial and is studied under the name of \emph{actual causation}~\citep{Pearl2009Causality,Halpern2016ActualCausality}.
Humans would answer by contrasting the actual outcome to some normative world in which $A=a$ did not happen, \ie they would ask the counterfactual question ``What would have happened normally to $B$ without $A=a$?''~\cite{Halpern2016ActualCausality}.
Algorithmitizing this approach poses certain problems.
First, it requires a ``normal'' outcome which can be difficult to compute as it depends on the behavior of the different actors in the world.
Second, it requires to actually observe the world's state without the agents interference.
Such a ``no influence'' action may not be available for every agent.
Instead, we are inspired by an alternative approach, the so-called \emph{``but-for''} test: ``$B=b$ would not have happened but for $A=a$.''
In other words, $A=a$ was a necessary condition for $B=b$ to occur, and under a different value for $A$, $B$ would have had a different value as well.
This matches well with an algorithmic view on causation: $A$ is a cause of $B$ if the value of $A$ is required to determine the value of $B$~\cite{Pearl2016CausalInf}.

The but-for test yields potentially counterintuitive assessments.
Consider a robotic arm close to an object but performing an action that moves it away from the object.
Then this action is considered a cause for the position of the object in this step, as an alternative action touching the object would have led to a different outcome.
Algorithmically, knowing the action is required to determine what happens to the object -- all actions are considered to be a cause in this situation.
Importantly, this implies that we cannot differentiate whether individual actions are causes or not, but can only identify whether or not the agent has causal influence on other entities in the current state.

\section{Causal Influence Detection}
\label{sec:method}

As the previous discussion showed, having causal influence is dependent on the situation the agent is in, rather than the chosen actions.
We characterize this as the agent being \emph{in control}, analogous to similar notions in control theory~\cite{Touchette2004InfoControlSystems}.
Formally, using the causal model introduced in \sec{sec:background}, \emph{we define the agent to be in control of $S_j'$ in state $S=s$ if there is an edge $A \rightarrow S_j'$ in the local causal graph $\mathcal{G}_{S=s}$ under all interventions $\Do{A}{\pi(a | s)}$ with $\pi$ having full support.}
The following proposition states when such an edge exists (proofs in \supp{app:proofs_propositions}).

\begin{proposition}\label{prop:independence}
Let $\mathcal{G}_{S=s}$ be the graph of the local CGM induced by $S = s$.
There is an edge $A \rightarrow S_j'$ in $\mathcal{G}_{S=s}$ under the intervention $\Do{A}{\pi(a | s)}$ if and only if  ${S_j' \nindep A \mid S = s}$.
\end{proposition}

To detect when the agent is in control, we can intervene with a policy.
The following proposition gives conditions under which conclusions drawn from one policy generalize to many policies.

\begin{proposition}\label{prop:detection}
If there is an intervention $\Do{A}{\pi(a | s)}$ under which ${S_j' \nindep A \mid S = s}$, this dependence holds under \emph{all} interventions with full support, and the agent is in control of $S_j'$ in $s$.
If there is an intervention $\Do{A}{\pi(a | s)}$ with $\pi$ having full support under which ${S_j' \indep A \mid S = s}$, this independence holds under \emph{all} possible interventions and the agent is not in control of $S_j'$ in $s$.
\end{proposition}

\subsection{Measuring Causal Action Influence}
\label{subsec:measuring_cai}

Our goal is to find a state-dependent quantity that measures whether the agent is in control of $S_j'$.
As Prop.~\ref{prop:independence} tells us, control (or its absence) is linked to the independence $S_j' \indep A \mid S = s$.
A well-known measure of dependence is the conditional mutual information (CMI)~\citep{Cover2006Elements} which is zero for independence.
We thus propose to use (pointwise) CMI as a measure of \emph{causal action influence} (CAI) that can be thresholded to get a classification of control (see \supp{app:cmi_derivation} for a derivation):
\begin{align}
    \CAI(s) \colonequal I(S_j'; A \mid S = s) = \Exp_{a \sim \pi} \g[ \KL{P_{S_j' \mid s, a}}{P_{S_j' \mid s}} \g]. \label{eq:cai}
\end{align}
We want this measure to be independent of the particular policy used in the joint distribution $P(S, A, S')$.
This is because we might not be able to sample from or evaluate this policy (\eg in off-policy RL, the data stems from a mixture of different policies).
Fortunately, Prop.~\ref{prop:detection} shows that to detect control, it is sufficient to demonstrate (in-)dependence for a single policy with full support.
Thus, we can choose a uniform distribution over the action space as the policy: $\pi(A) \colonequal \mathcal{U}(\mathcal{A})$.

Let us discuss how CAI relates to previously suggested measures of (causal) influence.
\emph{Transfer entropy}~\cite{Schreiber2000TransferEnt} is a non-linear extension of Granger causality~\cite{Granger1969Causality} quantifying causal influence in time series under certain conditions.
CAI is similar to a one-step, local transfer entropy~\cite{Lizier2013InfoDynamics} with the difference that CAI conditions on the full state $S$.
\citet{Janzing2013CausalInf} put forward a measure of \emph{causal strength} fulfilling several natural criteria that other measures, including transfer entropy, fail to satisfy.
In \supp{app:proof_causal_strength}, we show that CAI is a pointwise version of \citeauthor{Janzing2013CausalInf}'s causal strength, for policies not conditional on the state $S$ (adding further justification for the choice of a uniform random policy).
Furthermore, we can relate CAI to notions of \emph{controllability}~\cite{Touchette2004InfoControlSystems}.
Decomposing $\CAI(s)$ as ${H(S_j' \mid s)} - H(S_j' \mid A, s)$, where $H$ denotes the conditional entropy~\cite{Cover2006Elements}, we can interpret CAI as quantifiying the degree to which $S_j'$ can be controlled in $s$, accounting for the system's intrinsic uncertainty that cannot be reduced by the action.

In the context of RL, \emph{empowerment}~\cite{Klyubin2005Empowerment,Salge2014Empowerment,Mohamed2015Empowerment} is a well-known quantity used for intrinsically-motivated exploration that leads agents to states of maximal influence over the environment.
Empowerment, for a state $s$, is defined as the channel capacity between action and a future state, which coincides with $\max_\pi C(s)$ for one-step empowerment.
CAI can thus be seen as a non-trivial lower bound of empowerment that is easier to compute.
However, CAI differs from empowerment in that it does not treat the state space as monolithic and is specific to an entity.
In \sec{subsec:app_exploration_bonus}, we demonstrate that an RL agent maximizing CAI quickly achieves control over its environment.

\subsection{Learning to Detect Control}
\label{subsec:learning_control}

Estimating CMI is a hard problem on many levels: it involves computing high dimensional integrals, representing complicated distributions and having access to limited data; strictly speaking, each conditioning point $s$ is seen only once in continuous spaces.
In practice, one thus has to resort to an approximation.
Non-parametric estimators based on nearest neighbors~\cite{Kraskov2004MI,Runge2018CondIndTest} or kernels methods~\cite{Kandasamy2015NonparamEst} are known to not scale well to higher dimensions~\cite{Gao2015EstimatingMI}.
Instead, we approach the problem by learning neural network models with suitable simplifying assumptions.

Expanding the KL divergence in \eqn{eq:cai}, we can write CAI as
\begin{align}
 \CAI(s) = I(S_j'; A \mid S = s) &= \Exp_{A \mid s} \Exp_{S_j' \mid s, a} \left[ \log \frac{p(s_j' \mid s, a)}{\int p(s_j' \mid s, a)\,\pi(a) \dx{a}} \right]
 \label{eq:cai2}
\end{align}
To compute this term, we estimate the transition distribution $p(s_j' \mid s, a)$ from data.
We then approximate the outer expectation and the transition marginal $p(s_j' \mid s)$ by sampling $K$ actions from the policy $\pi$.
This gives us the estimator
\begin{align}
\hat{\CAI}(s) = \frac{1}{K} \sum_{i=1}^K \left[ \KL{p(s_j' \mid s, a^{(i)})}{\frac{1}{K} \sum_{k=1}^K p(s_j' \mid s, a^{(k)})} \right], \label{eq:cmi_approx}
\end{align}
with $\{ a^{(1)}, \ldots, a^{(K)} \} \stackrel{\text{iid}}{\sim} \pi$.
Here, we replaced the infinite mixture $p(s_j' \mid s)$ with a finite mixture, $p(s_j' \mid s) \approx \frac{1}{K} \sum_{i=1}^K p\g(s_j' \mid s, a^{(i)}\g)$, and used Monte-Carlo to approximate the expectation.
\citet{Poole2019VarBoundsMI} show that this estimator is a lower bound converging to the true mutual information $I(S_j'; A \mid S = s)$ as $K$ increases (assuming, however, the true density $p(s_j' \mid s, a)$).

To compute the KL divergence itself, we make the simplifying assumption that the transition distribution $p(s_j' \mid s, a)$ is normally distributed given the action, which is reasonable in the robotics environment we are targeting.
This allows us to estimate the KL without expensive MC sampling by using an approximation for mixtures of Gaussians from~\citet{Durrieu2012LowerUpperKL}.
We detail the exact formula we use in \supp{app:approx_kl_gaussians}.

With the normality assumption, the density itself can be learned using a probabilistic neural network and simple maximum likelihood estimation.
That is, we parametrize $p(s_j' \mid s, a)$ as $\mathcal{N}(s_j'; \mu_\theta(s, a), \sigma^2_\theta(s, a))$, where $\mu_\theta, \sigma^2_\theta$ are the outputs of a neural network $f_\theta(s, a)$.
We find the parameters $\theta$ by minimizing the negative log-likelihood over samples $\mathcal{D}=\{ (s^{(i)}, a^{(i)}, s'^{(i)}) \}_{i=1}^N$ collected by some policy (the univariate case shown here also extends to the multivariate case):
\begin{align}
    \theta^\ast = \argmin_\theta \frac{1}{N} \sum_{i=1}^N \frac{\big( s'^{(i)}_j - \mu_\theta\big(s^{(i)}, a^{(i)}\big)\big)^2}{2 \sigma_\theta^2\big(s^{(i)}, a^{(i)}\big)} + \frac{1}{2} \log \sigma_\theta^2\big(s^{(i)}, a^{(i)}\big). \label{eq:log_likelihood}
\end{align}
There are some intricacies regarding the policy that collects the data for model training and the sampling policy $\pi$ that is used to compute CAI.
First of all, the two policies need to have overlapping support to avoid evaluating the model under actions never seen during training.
Furthermore, if the data policy is different from the sampling policy $\pi$, the model is biased to some degree.
This suggests to use $\pi$ for collecting the data; however, as we use a random policy, this will not result in interesting data in most environments.
The bias can be reduced by sampling actions from $\pi$ during data collection with some probability and only train on those.
In practice, however, we find to obtain better performing models by training on all data despite potentially being biased.

\section{Empirical Evaluation of Causal Influence Detection}
\label{sec:empirical_eval}

In this section, we evaluate the quality of our proposed causal influence detection approach in relevant environments.
As a simple test case, we designed an environment (\Slide) in which the agent must slide an object to a goal location by colliding with it.
Furthermore, we test on the \FPP environment from OpenAI Gym~\cite{Brockman2016OpenAIG}, in its original setting and when adding Gaussian noise to the observations to simulate more real-world conditions.
In both environments, the target variables of interest are the coordinates of the object.
Note that we need the true causal graph at each time step for the evaluation.
For \Slide, we derive this information from the simulation.
For the pick and place environment with its non-trivial dynamics, we resort to a heuristic of the possible movement range of the robotic arm in one step.
Detailed information about the setup is provided in Suppls.~\ref{app:environments} and \ref{app:setup_influence_eval}.

For our method, we use CAI estimated according to \eqn{eq:cmi_approx} (with $K=64$) as a classification score that is thresholded to gain a binary decision.
We compare with a recently proposed method~\cite{Pitis2020CFDataAug} that uses the attention weights of a Transformer model~\cite{Vaswani2017Transformer} to model influence.
Moreover, we compare with an \emph{Entropy} baseline that uses $H(S'_j \mid s)$ as a score and a \emph{Contact} baseline based on binary contact information from the simulator.
We show the test results over 5 random seeds in \tab{tab:inf_eval_results} and \fig{fig:inf_eval_roc_curves}.
We observe that CAI is able to reliably detect causal influence and no other baseline is able to do so. %
When increasing the observation noise, the performance drops gracefully for CAI as shown in \fig{fig:inf_eval_roc_curves}c.
\supp{app:add_results_inf_eval} contains more experimental results, including a visualization of CAI's behavior.

\begin{figure}
  \begin{center}
    \begin{tabular}{@{}c@{\hskip 1em}c@{\hskip 1em}c@{}}
    \hspace{1.2em}(a) \small\Slide &\hspace{1.2em}(b) \small\FPP & \hspace{1.2em}(c) \small\FPP\\[-0.2em]
    \includegraphics[width=0.30\textwidth]{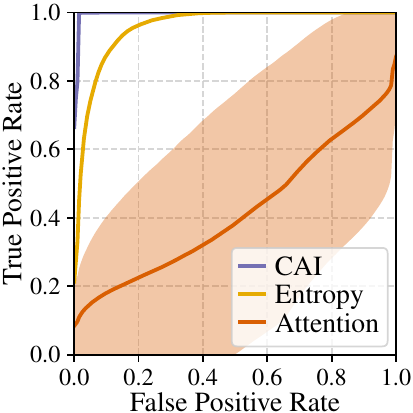}&
    \includegraphics[width=0.30\textwidth]{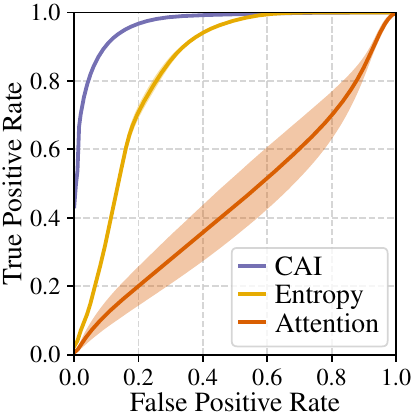}&
    \includegraphics[width=0.30\textwidth]{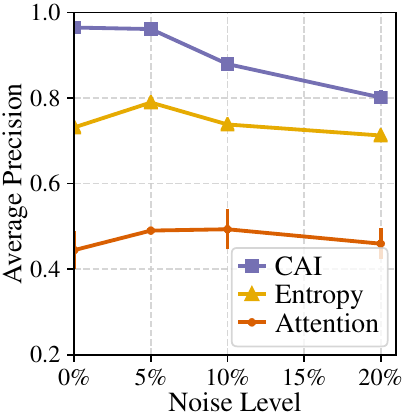}
    \end{tabular}
    \caption{Causal influence detection performance. (a, b) ROC curves on \Slide and \FPP environments. (c) Average precision for \FPP depending on added state noise. Noise level is given as percentage of one standard deviation over the dataset.}
    \label{fig:inf_eval_roc_curves}
    \end{center}
\end{figure}

\begin{table}
  \caption{Results for evaluating causal influence detection on different environments. We measure area under the ROC curve (AUC), average precision (AP), and the best achievable F-score (F\textsubscript{1}).}
  \label{tab:inf_eval_results}
  \vspace{.3em}
  \centering
  \renewcommand{\arraystretch}{1.03}
  \resizebox{.98\linewidth}{!}{ %
  \begin{tabular}{@{}rlllllll@{}}
    \toprule
    & \multicolumn{3}{c}{\textbf{\Slide}} & \phantom{} & \multicolumn{3}{c}{\textbf{\FPP}} \\
    \cmidrule{2-4} \cmidrule{6-8}
                                          & \multicolumn{1}{c}{AUC} & \multicolumn{1}{c}{AP} & \multicolumn{1}{c}{F\textsubscript{1}} &  & \multicolumn{1}{c}{AUC} & \multicolumn{1}{c}{AP} & \multicolumn{1}{c}{F\textsubscript{1}} \\
    \midrule
    CAI (ours)                            & $1.00\pm 0.00$   & $0.98\pm 0.00$   &  $0.95\pm 0.01$ &&  $0.97\pm 0.01$   & $0.96\pm 0.00$   &  $0.89\pm 0.00$                                                                                          \\
    Entropy                     &$0.96\pm 0.00$   & $0.47\pm 0.01$   &  $0.50\pm 0.01$ &&  $0.84\pm 0.00$   & $0.73\pm 0.00$   &  $0.78\pm 0.00$     \\
    Attention~\cite{Pitis2020CFDataAug} & $0.42\pm 0.31$   & $0.13\pm 0.14$   &  $0.18\pm 0.17$ &&  $0.46\pm 0.06$   & $0.44\pm 0.04$   &  $0.62\pm 0.00$    \\
    Contacts                     & $0.89$   & $0.78$   &  $0.88$ &&  $0.79$   & $0.77$   &  $0.73$     \\
    \bottomrule
  \end{tabular}
 }
\end{table}

\section{Improving Efficiency in Reinforcement Learning}
\label{sec:rl_applications}

Having established the efficacy of our causal action influence (CAI) measure, we now develop several approaches to use it to improve RL algorithms.%
We will empirically verify the following claims in robotic manipulation environments: CAI improves sample efficiency and performance by (i) better state exploration through an exploration bonus,
(ii) causal action exploration, and (iii) prioritizing experiences with causal influence during training.

We consider the environments \FP, \FPP from OpenAI Gym~\cite{Plappert2018MultiGoalRL}, and \FRT which is our modification containing a rotating table (explained in \supp{app:env_fetch_rot_table}).
These environments are goal-conditioned RL tasks with sparse rewards, meaning that each episode, a new goal is provided and the agent only receives a distinct reward upon reaching it.
We use DDPG~\cite{Lillicrap2016DDPG} with hindsight experience replay (HER)~\cite{Andrychowicz2017HER} as the base RL algorithm, a combination that achieves state-of-the-art results in these environment.
The influence detection model is trained online on the data collected from an RL agent learning to solve its task.
Since our measure \CAI requires an entity of interest, we choose the coordinates of the object (as $S_j$).
In all experiments, we report the mean success rate with standard deviation over 10 random seeds.
More information about the experimental settings can be found in \supp{app:hyperparameters_rl}.

\subsection{Intrinsic Motivation to Seek Influence}\label{subsec:app_exploration_bonus}

\textbf{Causal Action Influence as Reward Bonus.}
We hypothesize that it is useful for an agent to be intrinsically motivated to gain control over its environment.
We test this hypothesis by letting the agent maximize the causal influence it has over entities of interest.
This can be achieved by using our influence measure as a reward signal.
The reward signal can be used on its own, as an intrinsic motivation-type objective, or in conjunction with a task-specific reward as an exploration bonus.
In the former case, we expect the agent to discover useful behaviors that can help it master task-oriented skills afterwards; in the latter case, we expect learning efficiency to improve, especially in sparse extrinsic reward scenarios.
Concretely, for a state $s$, we define the bonus as $r_{\mathrm{CAI}}(s)=\CAI(s)$, and the total reward as $r(s) = r_{\mathrm{task}}(s) + \lambda_{\mathrm{bonus}}\ r_{\mathrm{CAI}}(s)$, where $r_{\mathrm{task}}(s)$ is the task reward, and $\lambda_{\mathrm{bonus}}$ is a hyperparameter.

\begin{figure}
  \centering
  \begin{minipage}[t]{.48\textwidth}
    \centering
    \includegraphics[width=0.9792\textwidth]{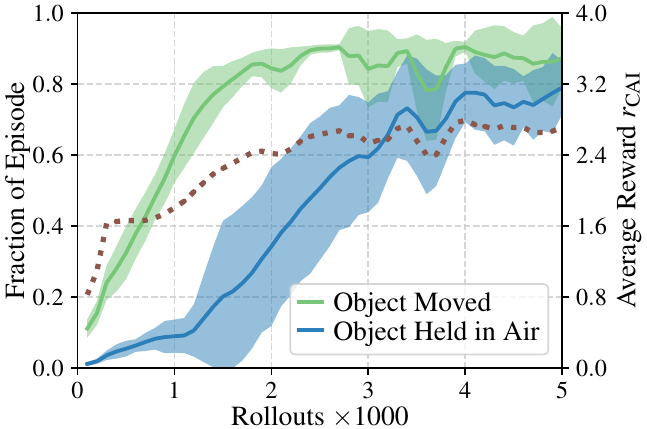}%
    \captionof{figure}{Intrinsically motivated learning on \FPP. The reward is only $r_{\mathrm{CAI}}$ measured on the object coordinates.}
    \label{fig:intrinsic_motivation}
  \end{minipage}%
  \hfill
  \begin{minipage}[t]{.48\textwidth}
    \centering
    \includegraphics[width=0.875\textwidth]{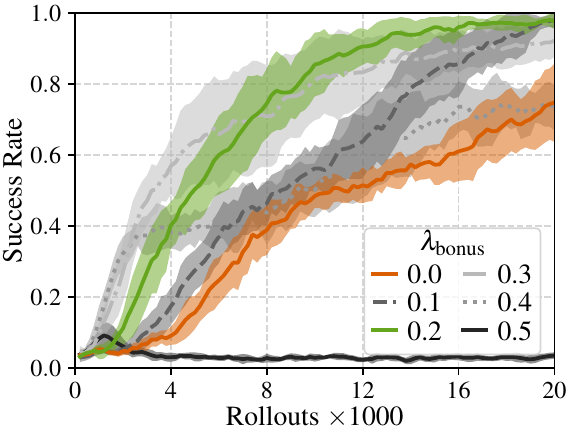}%
    \captionof{figure}{\emph{Exploration bonus} improves performance in \FPP.
      Sensitivity to the bonus reward scale $\lambda_\text{bonus}$.}
    \label{fig:bonus_ablation}
  \end{minipage}
\end{figure}

\textbf{Experiment on Intrinsically Motivated Learning.}
We first test the behavior of the agent in the absence of any task-specific reward on the \FPP environment.
Interestingly, the agent learned to grasp, lift, and hold the object in the air already after 2000 episodes, as shown in \fig{fig:intrinsic_motivation}.
The results demonstrate that encouraging causal control over the environment is well suited to prepare the agent for further tasks it might have to solve.

\textbf{Impact of CAI Reward Bonus.}
Second, we are interested in the impact of adding an exploration bonus.
In \fig{fig:bonus_ablation}, we present results on the \FPP environment when varying the reward scale $\lambda_\text{bonus}$.
Naturally, the exploration bonus needs to be selected in the appropriate scale as a value too high will make it dominate the task reward. 
If selected correctly, the sample efficiency is improved drastically; for example, we find that the agent reaches a success rate of 60\% four-times faster than the baseline (DDPG+HER) without any bonus ($\lambda_\text{bonus}=0$).

\subsection{Actively Exploring Actions with Causal Influence}\label{sec:active_explore}

\textbf{Following Actions with the Most Causal Influence.}
Exploration via bonus rewards favors the re-visitation of already seen states.
An alternative approach to exploration uses pro-active planning to choose exploratory actions.
In our case, we can make use of our learned influence estimator to pick actions which we expect will have the largest causal effect on the agent's surroundings.
From a causal viewpoint, the resulting agent can be seen as an experimenter that performs planned interventions in the environment to verify its beliefs.
Should the actual outcome differ from the expected outcome, subsequent model updates can integrate the new data to self-correct the causal influence estimator.

Concretely, given the agent being in state $s$, we choose the action that has the largest contribution to the empirical mean in \eqn{eq:cmi_approx}:
\begin{align}\label{eq:active_expl}
a^\ast = \argmax_{a \in \{a^{(1)}, \dots, a^{(K)}\} } \KL{p(s_j' \mid s, a)}{\frac{1}{K} \sum_{k=1}^K p(s_j' \mid s, a^{(k)})},
\end{align}
with $\{a^{(1)}, \ldots, a^{(K)}\} \stackrel{\text{iid}}{\sim} \pi$.
Intuitively, the selected action will be the one which results in maximal deviation from the expected outcome under all actions.
For states $s$ where the the agent is not in control, \ie $\CAI(s) \approx 0$, the action selection is uniform at random.

\textbf{Active Exploration in Practice.}
Can active exploration replace $\epsilon$-greedy exploration?
To gain insights, we study the impact of the fraction of actively chosen exploration actions.
For every exploratory action ($\epsilon$ is 30\% in our experiments), we choose an action according to \eqn{eq:active_expl} the specified fraction of the time, and otherwise a random action.
\Fig{fig:active_expl_ablation} shows that any amount of active exploration improves over simple random exploration.
Active causal action exploration can improve the sample efficiency roughly by a factor of two.
\begin{figure}
  \centering
  \hfill
  \begin{minipage}[t]{.48\textwidth}
    \centering
    \includegraphics[width=0.875\textwidth]{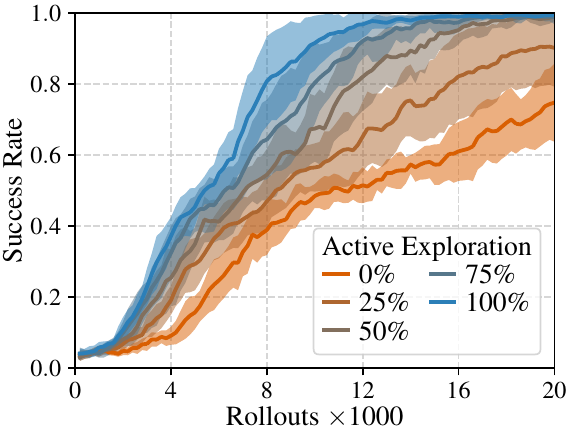}%
    \captionof{figure}{Performance of \emph{active exploration} in \FPP
      depending on the fraction of exploratory actions chosen actively \eqnp{eq:active_expl} from a total of 30\% exploratory actions.
    }
    \label{fig:active_expl_ablation}
  \end{minipage}
  \hfill
  \begin{minipage}[t]{.48\textwidth}
    \centering
    \includegraphics[width=0.875\textwidth]{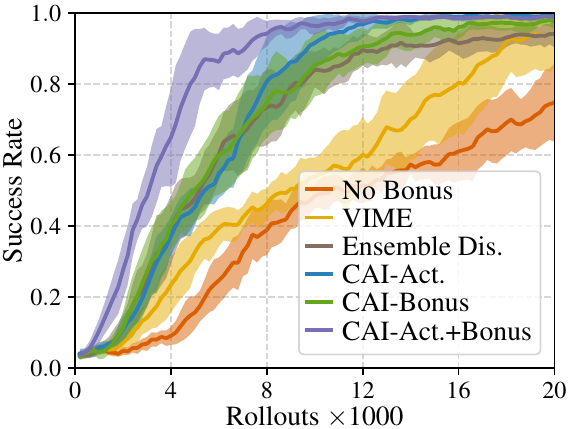}%
    \captionof{figure}{Experiment comparing exploration strategies on \FPP.
      The combination of active exploration and reward bonus yields the largest sample efficiency.
    }
    \label{fig:exploration}
  \end{minipage}
  \hfill
\end{figure}

\textbf{Combined CAI Exploration.}
We also present the combination of reward bonus and active exploration and compare our method with VIME, another exploration scheme based on information-theoretic measures~\cite{Houthooft2016VIME}.
In contrast to our method, VIME maximizes the information gain about the state transition dynamics.
Further, we compare to ensemble disagreement~\cite{Pathak2019ExplorationDisagreement}, which in effect minimizes epistemic uncertainty about the transition dynamics.
We compare different variants of VIME and ensemble disagreement in \supp{app:add_results_rl}, and display only their best versions here.
\Fig{fig:exploration} quantifies the superiority of all CAI variants (with ensemble disagreement as a viable alternative) and shows that combining the two exploration strategies compounds to increase sample efficiency even further.
In the figure, CAI uses 100\% active exploration and $\lambda_\mathrm{bonus}=0.2$ as the bonus reward scale.

\subsection{Causal Influence-based Experience Replay} \label{subsec:app_exp_replay}

\textbf{Prioritizing According to Causal Influence.}
We will now propose another method using CAI, namely to inform the choice of samples replayed to the agent during off-policy training.
Typically, past states are sampled uniformly for learning.
Intuitively, focusing on those states where the agent has control over the object of interest (as measured by CAI) should improve the sample efficiency.
We can implement this idea using a prioritization scheme that samples past episodes in which the agent had more influence more frequently.
Concretely, we define the probability $P^{(i)}$ of sampling any state from episode $i$ (of $M$ episodes) in the replay buffer as
\begin{align*}
\refstepcounter{equation}\latexlabel{firsthalf}
\refstepcounter{equation}\latexlabel{secondhalf}
P^{(i)} = \frac{p^{(i)}}{\sum_{i=1}^{M} p^{(i)}} \cdot \frac{1}{T}, && \text{with} && p^{(i)} &= \bigg(M + 1 - \mathrm{\rank_i} \sum_{t=1}^T \CAI\big(s^{(t)} \big) \bigg)^{\!-1}\hspace*{-.8em}.
\tag{\ref*{firsthalf}, \ref*{secondhalf}}
\end{align*}
where $T$ is the episode length, and $p^{(i)}$ is the priority of episode $i$.
The priority of an episode $i$ is based on the (inverse) rank of the episode ($\mathrm{\rank_i}$) when sorting all $M$ episodes according to their total influence (\ie sum of state influences). 
We call this \emph{causal action influence prioritization} (\CAIP).

This scheme is similar to Prioritized Experience Replay~\citep{Schaul2016PER}, with two differences:
instead of using the TD error for prioritization, we use the causal influence measure.
Furthermore, instead of prioritizing individual states, we prioritize episodes and sample states uniformly within episodes.
This is because the information about the return that can be achieved from an influence state still needs to be propagated back to non-influenced states by TD updates, which requires sampling them.

\begin{figure}
    \begin{center}
    \includegraphics[width=0.33\textwidth]{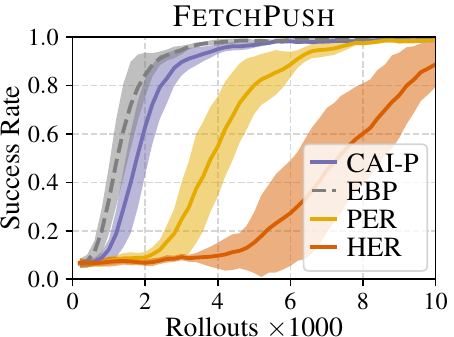}%
    \includegraphics[width=0.33\textwidth]{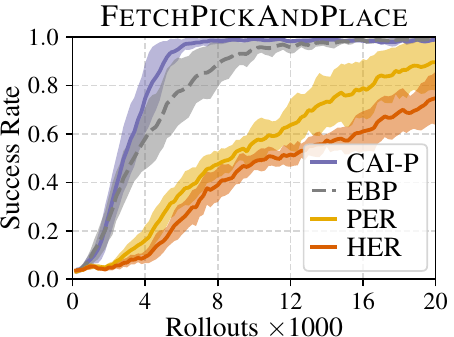}%
    \includegraphics[width=0.33\textwidth]{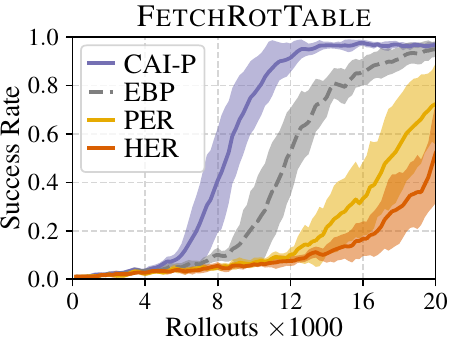}
    \caption{Prioritizing experience replay in different manipulation environments.
      Comparison of causal action influence prioritization (\CAIP) against baselines: the energy-based method (EBP)~\cite{Zhao2018EBP} with privileged information, prioritized experience replay (PER)~\citep{Schaul2016PER}, and HER without prioritization.}
    \label{fig:exp_replay}
    \end{center}
\end{figure}

\begin{wrapfigure}{r}{.19\textwidth}
   \vspace{1em}
   \includegraphics[width=\linewidth]{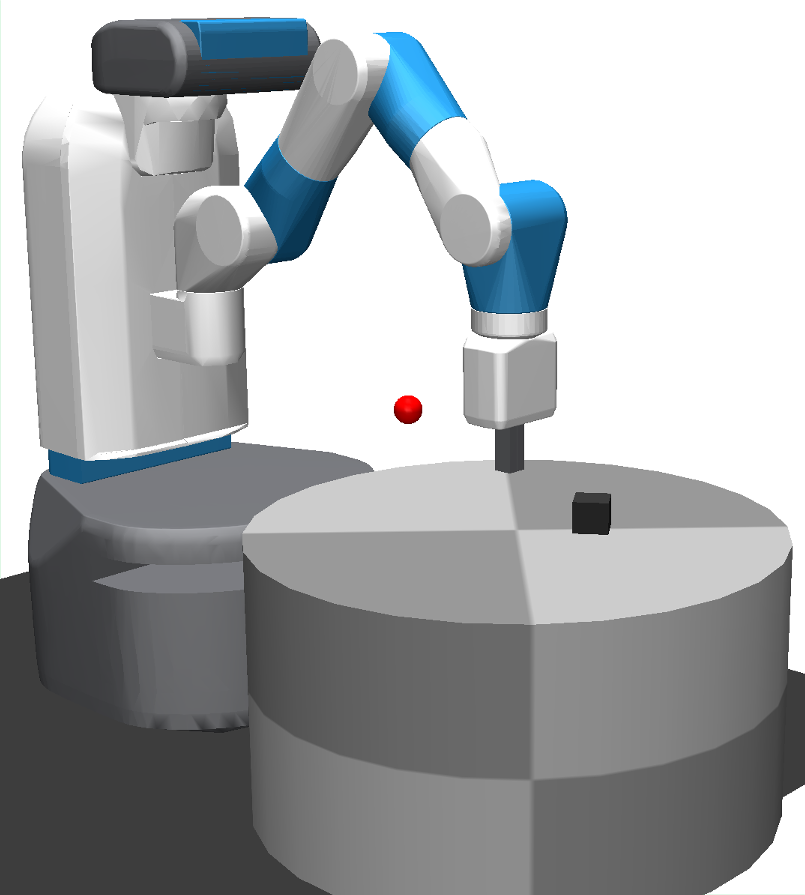}
   \caption{\FRT. The table rotates periodically.}\label{fig:rot-table}
\end{wrapfigure}
\textbf{Influence-Based Prioritization in Manipulation Tasks.} We compare our influence-based prioritization (\CAIP) against no prioritization in hindsight experience replay (HER) (a strong baseline for multi-goal RL), and two other prioritization schemes: prioritized experience replay (PER)~\citep{Schaul2016PER} and energy-based prioritization (EBP)~\cite{Zhao2018EBP}.
Especially EBP is a strong method for the environments we are considering as it uses privileged knowledge of the underlying physics to replay episodes based on the amount of energy that is transferred from agent to the object to manipulate.
All prioritization variants are equipped with HER as well.
The \FRT environment, shown in \fig{fig:rot-table}, is an interesting test bed as the object can move through the table rotation without the control of the agent.
The results, shown in \fig{fig:exp_replay}, reveal that causal influence prioritization can speed up learning drastically. Our method is on par or better than the energy-based (oracle) method EBP and improves over PER by a factor of 1.5--2.5 in learning speed (at 60\% success rate).
Finally, in \supp{app:add_results_rl}, we combine all our proposed improvements and show that \FPP can be solved up to 95\% success rate in just 3000 episodes.

\section{Discussion}
\label{sec:discussion}

In this work, we show how situation-dependent causal influence detection can help improve reinforcement learning agents.
To this end, we derive a measure of local causal action influence (CAI) and introduce a data-driven approach based on neural network models to estimate it.
We showcase using CAI as an exploration bonus, as a way to perform active action exploration, and to prioritize in experience replay.
Each of our applications yields strong improvements in sample efficiency.
We expect that there are further ways to use our causal measure in RL, \eg for credit assignment.

Our work has several limitations.
First, we assume full observability of the state, which simplifies the causal inference problem as there is no confounding between an agent's action and its effect.
Under partial observability, our approach could still be applicable using latent variable models~\cite{Louizos2017CausalEffect}.
Second, we require an available factorization of the state into causal variables.
The problem of automatically learning causal variables from high-dimensional data is open \cite{Schoelkopf2021CausalReprLearning} and our method would likely benefit from advances in this field.
Third, the accurate estimation of our measure relies on a correct model.
We found that deep networks can struggle at times to pick up the causal relationship between actions and entities.
How to design models with appropriate inductive biases for cause-effect inference is an open question~\cite{Peters2017CausalInference,Goyal2020InductiveBias,Schoelkopf2021CausalReprLearning}.

An intriguing future direction is to extend our work to influence detection between entities, a prerequisite for identifying multi-step influences of the agent on the environment.
Being able to model such indirect interventions would bring us closer to ``artificial scientists'' -- agents that can perform planned experiments to reveal the latent causal structure of the world.

\begin{ack}
The authors thank Andrii Zadaianchuk and Dominik Zietlow for many helpful discussions and providing feedback on the text.
Furthermore, the authors would also like to thank Sebastian Blaes for creating the \FRT environment.
The authors thank the International Max Planck Research School for Intelligent Systems (IMPRS-IS) for supporting Maximilian Seitzer. 
GM and BS are members of the Machine Learning Cluster of Excellence, EXC number 2064/1 -- Project number 390727645. 
We acknowledge the financial support from the German Federal Ministry of Education and Research (BMBF) through the Tübingen AI Center (FKZ: 01IS18039B). 
\end{ack}

\bibliography{neurips_2021}

\newpage
\appendix

\renewcommand{\thetable}{S\arabic{table}}
\renewcommand{\thefigure}{S\arabic{figure}}
\setcounter{figure}{0}
\setcounter{table}{0}

\begin{center}\large \textbf{Supplementary Material}\end{center}

\section{Proofs and Derivations}
\label{app:proofs}

\subsection{Proof of Propositions}
\label{app:proofs_propositions}

We first clarify the behavior of local CGMs (see Def.~\ref{def:local_cgm}) under interventions. 
In particular, the local CGM induced by $X=x$ on $P(\mathcal V)$ under an intervention $\Do{V}{v}$ is defined to be the local CGM induced by $X=x$ on the joint distribution $P^\Do{V}{v}(\mathcal V)$.

\begin{proof}[Proof of Prop.~\ref{prop:independence}]
``if'': If it holds that $S_j' \nindep A \mid S = s$, then by the Markov property~\cite[Def.~6.21]{Peters2017CausalInference}, there must be an unblocked path from $A$ to $S_j'$ in $\mathcal{G}_{S=s}$. Because the path over $S$ is blocked by observing $S=s$, and we assume no instantaneous effects, the only possible such path is the direct edge $A \rightarrow S_j'$.\\
``only if'': If there is an edge $A \rightarrow S_j'$ in $\mathcal{G}_{S=s}$, then causal minimality of the local causal graph says that $S_j'$ must be dependent on each parent given its other parents, meaning $S_j' \nindep A \mid S = s$.
\end{proof}

\begin{proof}[Proof of Prop.~\ref{prop:detection}]
To show the first part, note that the dependence ${S_j' \nindep A \mid S = s}$ under $\Do{A}{\pi(a | s)}$ implies that there exists some $s_j'$ and $a_1,\ a_2$ with $\pi(a_1 \mid s) > 0,\ \pi(a_2 \mid s) > 0$ for which
\begin{align}
p^\Do{A}{\pi}(s_j' \mid s, a_1) = p(s_j' \mid s, a_1) \neq p(s_j' \mid s, a_2) = p^\Do{A}{\pi}(s_j' \mid s, a_2).
\end{align}
Any $\pi'$ with full support would also have $\pi'(a_1 \mid s) > 0,\ \pi'(a_2 \mid s) > 0$, and so
\begin{align}
p^\Do{A}{\pi'}(s_j' \mid s, a_1) = p(s_j' \mid s, a_1) \neq p(s_j' \mid s, a_2) = p^\Do{A}{\pi'}(s_j' \mid s, a_2),
\end{align}
implying the dependence under $\Do{A}{\pi'}$.
To show that the agent is in control of $S_j'$ in $S=s$, there needs to be an edge $A \rightarrow S_j'$ in $\mathcal{G}_{S=s}$ under all interventions $\Do{A}{\pi'}$ with $\pi'$ having full support.
This is the case, because as shown, for all interventions $\Do{A}{\pi'}$ with $\pi'$ having full support it holds that ${S_j' \nindep A \mid S = s}$, and by Prop.~\ref{prop:independence}, there is an edge $A \rightarrow S_j'$ in $\mathcal{G}_{S=s}$ under any such intervention.

To show the second part, we show that if ${S_j' \indep A \mid S = s}$ under any intervention $\Do{A}{\pi}$ with $\pi$ having full support, for any intervention $\Do{A}{\pi'}$ it holds that $P^{\Do{A}{\pi'}}(S_j' \mid S = s,  A) = P^\Do{A}{\pi'}(S_j' \mid S = s)$.
This follows from the fact that for any $\pi'$, it holds that
\begin{align}
P^\Do{A}{\pi'}(S_j' \mid S = s, A) &= P(S_j' \mid S = s, A) \\
&= P(S_j' \mid S = s) = P^\Do{A}{\pi'}(S_j' \mid S = s)
\end{align}
where the first equality is due to the autonomy property of causal mechanisms~\cite[Eq.~6.7]{Peters2017CausalInference}, and the second equality because of the independence ${S_j' \indep A \mid S = s}$.
Note that if $\pi$ had not had full support, we would not be allowed to use the second equality as then $P(S_j' \mid S = s, A = a) = P(S_j' \mid S = s)$ only for $a$ with $\pi(a) > 0$.
The agent is not in control of $S_j'$ in $S=s$, as for all interventions $\Do{A}{\pi'}$ with $\pi'$ having full support, there is no edge $A \rightarrow S_j'$ in $\mathcal{G}_{S=s}$ by Prop.~\ref{prop:independence}.
\end{proof}

\subsection{CMI Formula}
\label{app:cmi_derivation}

\begin{align}
 \CAI(s) = I(S_j'; A \mid S = s) &= \KL{P_{S_j', A \mid s}}{P_{S_j' \mid s} \otimes P_{A \mid s}} \\
&= \Exp_{S_j', A \mid s} \left[ \log \frac{p(s_j', a \mid s)}{p(s_j' \mid s) \pi(a \mid s)} \right] \\
 &= \Exp_{S_j', A \mid s} \left[ \log \frac{p(s_j' \mid s, a)}{p(s_j' \mid s)} \right] \\
 &= \Exp_{A \mid s} \left[ \KL{P_{S_j' \mid s, a}}{P_{S_j' \mid s}} \right] \label{eq1}
\end{align}

\subsection{Proof that CAI is a Pointwise Version of Janzing et al.'s Causal Strength}
\label{app:proof_causal_strength}

\newcommand{\notx}{{\backslash\mkern-2mu X}}

\citeposs{Janzing2013CausalInf} measure of causal strength quantifies the impact that removing a set of arrows in the causal graph would have.
As it is the relevant case for us, we concentrate here on the single arrow version, for instance, between random variables $X$ and $Y$.
The idea is to consider the arrow as a ``communication channel'' and evaluate the corruption that could be done to the signal that flows between $X$ and $Y$ by cutting the channel.
To do so, the distribution that feeds the channel is replaced with $P(X)$, \ie the marginal distribution of $X$.
The measure of causal strength then is equal to the difference between the pre- and post-cutting joint distribution as given by the KL divergence.

Formally, let $\mathcal{V}$ denote the set of variables in the causal graph, let $X \rightarrow Y$ be the arrow of interest with $X, Y \in \mathcal{V}$, and let $\Pa_Y^\notx$ be the set of parents of $Y$ without $X$.
Then, the post-cutting distribution on $Y$ is defined as
\begin{align}
p_{X \rightarrow Y}\g(y \mid \pa_Y^\notx\g) = \int p\g(y \mid x, \pa_Y^\notx\g) p\g(x\g) \dx{x}.
\end{align}
The new joint distribution after such an intervention is defined as
\begin{align}
p_{X \rightarrow Y}\g(v_1, \dots, v_{|\mathcal{V}|} \g) = p_{X \rightarrow Y}\g(y \mid \pa_Y^\notx\g) \prod_{\stackrel{j}{v_j \neq y}} p\g(v_j \mid \Pa(v_j)\g).
\end{align}
Finally, the causal strength of the arrow $X \rightarrow Y$ is defined as
\begin{align}
\mathfrak{C}_{X \rightarrow Y} = \KL{P(V)}{P_{X \rightarrow Y}(V)}.
\end{align}

To show the correspondence to our measure, we first restate Lemma 3 of \cite{Janzing2013CausalInf}.
For a single arrow $X \rightarrow Y$, causal strength can also be written as the KL between the conditionals on $Y$:
\begin{align}
\mathfrak{C}_{X \rightarrow Y} &= \KL{P(Y \mid \Pa_Y)}{P_{X \rightarrow Y}\g(Y \mid \Pa_Y^\notx\g)} \\
 &= \int \KL{P(Y \mid \pa_Y)}{P_{X \rightarrow Y}\g(Y \mid \pa_Y^\notx\g)} P(\pa_Y) \dx{\pa_Y}
\end{align}
We rewrite this further:
\begin{align}
&= \int P\g(\pa_Y^\notx\g) P\g(X \mid \pa_Y^\notx\g) \KL{P(Y \mid \pa_Y)}{P_{X \rightarrow Y}\g(Y \mid \pa_Y^\notx\g)} \dx{x}\,\dx{\pa_Y^\notx} \\
&= \Exp_{\pa_Y^\notx} \g[ \int P\g(X \mid \pa_Y^\notx\g) \KL{P(Y \mid \pa_Y)}{P_{X \rightarrow Y}\g(Y \mid \pa_Y^\notx\g)} \dx{x} \g] \\
&= \Exp_{\pa_Y^\notx} \g[ \textcolor{ourorange}{\Exp_{x \mid \pa_Y^\notx} \g[ \KL{P(Y \mid \pa_Y)}{P_{X \rightarrow Y}\g(Y \mid \pa_Y^\notx\g)} \g]} \g]
\end{align}
And we see that the \textcolor{ourorange}{inner part} corresponds to our measure $\CAI(s)$ (cmp. \eqn{eq:cai}) for the choices of variables $X
\mathrel{\widehat{=}} A$, $Y \mathrel{\widehat{=}} S'_j$, and $\Pa_Y^\notx \mathrel{\widehat{=}} S$, provided that $P(X) \mathrel{\widehat{=}} \pi(A)$ is not dependent on $S$. Thus, CAI is a pointwise version of causal strength for policies independent of the state:
\begin{align}
\mathfrak{C}_{A \rightarrow S'_j} = \Exp_{s} \g[ \CAI(s) \g]\tag*{$\square$}
\end{align}

\subsection{Approximating the KL Divergence Between a Gaussian and a Mixture of Gaussians}
\label{app:approx_kl_gaussians}

In this section, we give the approximation we use for the KL divergence in \eqn{eq:cmi_approx}.
We first state the approximation for the general case of the KL between two mixtures of Gaussians, and then specialize to our case when the first distribution is Gaussian distributed.
Here, we use the notation of \citet{Durrieu2012LowerUpperKL}.

Let $f$ be the PDFs of a multivariate Gaussians mixture with $A$ components, mixture weights $\omega_a^f \in (0, 1]$, means $\mu_a^f \in \Real^d$ and covariances $\Sigma_a^f\g \in \Real^{d\times d}$, where $a \in \{1, \dots, A\}$ is the index is of the $a$'th component. Then,
\begin{align}
 f(x) = \sum_{a=1}^A \omega_a^f f_a(x) = \sum_{a=1}^A \omega_a^f \mathcal{N}\g(x; \mu_a^f, \Sigma_a^f\g),
\end{align}
where $f_a(x)=\mathcal{N}\g(x; \mu_a^f, \Sigma_a^f\g)$ is the PDF of a multivariate Gaussian with mean $\mu_a^f$ and covariance $\Sigma_a^f$.
Analously, let $g$ be the PDF of a multivariate Gaussians mixture with $B$ components.
We are interested in the KL divergence $\KL{f}{g} = \int_{\Real^d} f(x) \log \frac{f(x)}{g(x)} \dx{x}$, which is intractable.

There are several ways to approximate the KL based on the decomposition $\KLapp{KL} = H(f, g) - H(f)$, where $H(f, g)$ is the cross-entropy between $f$ and $g$ and $H(f)$ is the entropy of $f$.
We will state the so-called \emph{product} approximation $\KLapp{prod}$ and the \emph{variational} approximation $\KLapp{var}$~\cite{Hershey2007KLApprox}.
Starting with $\KLapp{prod}$:
\begin{align}
 H(f, g) &\geq -\sum_a \omega_a^f \log \g( \sum_b \omega_b^g t_{ab} \g) \label{eq:ineq_cross_ent} \\
 H(f) &\geq -\sum_a \omega_a^f \log \g( \sum_{a'} \omega_{a'}^f z_{aa'} \g) \label{eq:ineq_ent} \\
 \KLapp{prod} &\colonequal \sum_a \omega_a^f \log \g( \frac{\sum_{a'} \omega_{a'}^f z_{aa'}}{\sum_b \omega_b^g t_{ab}} \g),
\end{align}
where $t_{ab}=\int f_a(x) g_b(x) \dx{x},\ z_{aa'}=\int f_a(x) f_{a'}(x) \dx{x}$ are normalization constants of product of Gaussians, and the inequalities in \eqref{eq:ineq_cross_ent}, \eqref{eq:ineq_ent} are based on Jensen's inequality. For the variational approximation:
\begin{align}
 H(f, g) &\leq \sum_a \omega_a^f H(f_a) - \sum_a \omega_a^f \log \g( \sum_b \omega_b^g e^{-\KL{f_a}{g_b}} \g)\label{eq:ineq_cross_ent_var} \\
  H(f) &\leq \sum_a \omega_a^f H(f_a) - \sum_a \omega_a^f \log \g( \sum_{a'} \omega_{a'}^f e^{-\KL{f_a}{f_{a'}}} \g) \label{eq:ineq_ent_var}\\
 \KLapp{var} &\colonequal \sum_a \omega_a^f \log \g( \frac{\sum_{a'} \omega_{a'}^f e^{-\KL{f_a}{f_{a'}}}}{\sum_b \omega_b^g e^{-\KL{f_a}{g_b}}} \g),
\end{align}
where \eqref{eq:ineq_cross_ent_var}, \eqref{eq:ineq_ent_var} are based on solving variational problems.
It can be shown that the mean between $\KLapp{prod}$ and $\KLapp{var}$ is the mean of a lower and upper bound to $\KLapp{KL}$, with better approximation qualities~\cite{Durrieu2012LowerUpperKL}.
Consequently, we use $\KLapp{mean} \colonequal \frac{\KLapp{prod} + \KLapp{var}}{2}$ as the basis of our approximation.
We can simplify $\KLapp{mean}$ as in our case, we know that $f$ has only one component.
This means that we can compute $H(f)$ in closed form, and do not need to use the inequalitities \eqref{eq:ineq_ent}, \eqref{eq:ineq_ent_var}.
Approximating $H(f,g)$ with the mean of the lower bound \eqref{eq:ineq_cross_ent} and upper bound \eqref{eq:ineq_cross_ent_var},
\begin{align}
H_\text{mean}(f,g) \colonequal \frac{1}{2} \g( -\log \g( \sum_b \omega_b^g t_{ab} \g) + H(f) - \log \g( \sum_b \omega_b^g e^{-\KL{f_a}{g_b}} \g) \g),
\end{align}
we get the final formula we use
\begin{align}
\KLapp{mean} &\colonequal H_\text{mean}(f,g) - H(f) \\
&= -\frac{1}{2} \log \g( \sum_b \omega_b^g t_{ab} \g) - \frac{1}{2} \log \g( \sum_b \omega_b^g e^{-\KL{f_a}{g_b}} \g) - \frac{1}{2} H(f).
\end{align}
Note that this term can become negative, whereas the KL is non-negative.
In practice, we thus threshold $\KLapp{mean}$ at zero.
For completeness, we also state the entropy of a Gaussian
\begin{align}
 H(f) = \frac{1}{2} \log \g((2\pi e)^d \lvert \Sigma \rvert \g), \label{eqn:entropy_gaussian}
\end{align}
the log normalization constant for a product of Gaussians
\begin{align}
 \log t_{ab} = -\frac{d}{2} \log 2\pi - \frac{1}{2} \log \lvert \Sigma_a^f + \Sigma_b^g \rvert - \frac{1}{2} (\mu_b^g - \mu_a^f)^T
 (\Sigma_a^f + \Sigma_b^g)^{-1} (\mu_b^g - \mu_a^f),
\end{align}
and the KL between two Gaussians
\begin{align}
 \KL{f_a}{g_b} = -\frac{d}{2} + \frac{1}{2} \log \frac{\lvert \Sigma_a^f \rvert}{\lvert \Sigma_b^g \rvert} + \frac{1}{2} \Tr \g( (\Sigma_b^g)^{-1} \Sigma_a^f \g) + \frac{1}{2} (\mu_b^g - \mu_a^f)^T (\Sigma_b^g)^{-1} (\mu_b^g - \mu_a^f).
\end{align}
In our experiments, we assume independent dimensions, that is, we parametrize the covariance $\Sigma$ as a diagonal matrix.
With this, the above formulas can be further simplified.

\section{Environments}
\label{app:environments}

\subsection{1D-Slide}
\label{app:env_1d_slide}

\begin{wrapfigure}[13]{r}{.25\textwidth}
    \vspace{-\baselineskip}
    \centerline{%
      \setlength{\fboxsep}{0pt}%
      \fbox{\includegraphics[width=0.25\columnwidth]{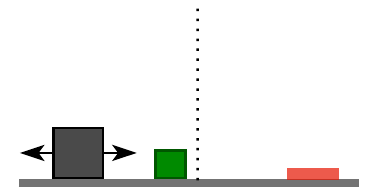}}
    }%
   \caption{Schematic of \Slide environment. The agent (black square) has to slide the object (green square) to the target location (red zone), but can not pass the dotted line.}
   \label{fig:oned_slide_env}
\end{wrapfigure}
\Slide is a simple environment that we designed to test influence detection, as we can easily derive when the agent has influence or not.
It consists of an agent and an object positioned on a line, with both agent and object only being able to move left and right.
See \fig{fig:oned_slide_env} for a visualization.
The goal of the agent is to move the object to a goal zone.
As the agent can not cross past the center of the environment, it has to hit the object at the appropriate speed.
The state space $\mathcal{S} \subset \mathbb{R}^4$ consists of position and velocity of the agent and object.
The agent's action $a \in \mathcal{A} \subseteq [-1, 1]$ applies acceleration to the agent.
On contact with the object, the full impulse of the agent is transferred to the object.
We can derive if the agent has causal influence in a state by simulating applying the maximum acceleration in both directions and checking whether a contact has occurred, and the object state has changed.

\subsection{FetchPickAndPlace}
\label{app:env_fetch_pickandplace}

The \textsc{Fetch} environments in OpenAI Gym~\cite{Brockman2016OpenAIG} are built on top of the MuJoCo simulator~\cite{Todorov2012MuJoCoAP}, and consist of a position-controlled 7 DoF robotic arm~\cite{Plappert2018MultiGoalRL}.
In \FPP, the state space $\mathcal{S} \subset \mathbb{R}^{25}$ consists of position and velocity of the robot's end effector, the position and velocity of the gripper, and the object's position and rotation, linear and angular velocities, and position relative to the end effector.
The action space $\mathcal{A} \subseteq [-1, 1]^{4}$ controls the gripper movement and opening/closening of the gripper.
For the experiments involving Transformer models, we split the state space into agent and object state components, where we do not include the relative positions between gripper and object into either component.

For our experiment in \sec{sec:empirical_eval} evaluating the causal influence detection, we need to determine whether the agent can potentially influence the object, \ie whether the agent is ``in control''.
This is difficult to determine analytically, which is why we designed a heuristic.
The idea is to find an ellipsoid around the end effector that captures its maximal movement range, and intersect this ellipsoid with the object to decide whether influence is possible.
As the range of movement of the robotic arm is different depending on its position, we first build a lookup table (offline) that contains, for different start positions, end positions after applying different actions.
To do so, we grid-sample the 3D-space over the table (50 grid points per dimension), plus a sparser sampling of the outer regions (20 grid points per dimension), resulting in \numprint{133000} starting locations for the lookup table.
Then, the following procedure is repeated for each starting location and action: after resetting the simulator, the end effector is manually moved to one of the starting locations, one of the maximal actions in each dimension (\ie, $-1$ and $1$, for a total of 6 actions) is applied, and the end position after one environment step is recorded in the lookup table.

Now, while the environment runs, for each step, we find the sampling point closest to the position of the robotic arm in the lookup table using a KD-tree, and find the corresponding end positions.
From the end positions, we build the ellipsoid by taking the maximum absolute deviation in each dimension to be the length of the ellipsoid axis in this dimension.
The ellipsoid so far does not take into account the spatial extents of object and gripper.
Thus, we extend the ellipsoid by the sizes of the object and gripper fingers in each dimension.
Furthermore, we take into account that the gripper fingers can be slid in y-direction by the agent's action by extending the ellipsoid's y-axis correspondingly.
The label of ``agent in control'' is then obtained by checking whether the object location lies within the ellipsoid.
Last, we also label a state as ``agent in control'' when there is an actual contact between gripper and object in the following step.
We note that the exact procedure described above is included in the code release.

\subsection{FetchRotTable}
\label{app:env_fetch_rot_table}

\FRT is an environment designed by us to test CAI prioritization in a more challenging setting.
In \FRT, the table rotates periodically, moving the object around.
This creates a confounding effect for influence detection, as there is another source of object movements besides the agent's own actions.
This means that CAI needs to differentiate between different causes for movements.

\FRT is based on \FPP, but the rectangular table is replaced with a circular one (see \fig{fig:rot-table}).
The table rotates with inertia by 45 degrees over the course of 25 environment steps, and then pauses for 5 steps.
To make the resulting environment Markovian, the state space of \FPP is extended to $\mathcal{S} \subset \mathbb{R}^{29}$, additionally including sine and cosine of the table angle, the rotational velocity of the table, and the amount the table will rotate in the current state in radians.
The task of the agent is the same as in \FPP, \ie to move the object to a target position.
If the target position is on the table, the agent thus has to learn to hold the object in position while the table rotates.
In contrast to \FPP, the goal distribution is different, with 90\% of goals in the air and only 10\% on the table.
This makes the task more challenging, as the agent has to master grasping and lifting before 90\% of the possible positive rewards can be accessed.

\section{Additional Results for Influence Evaluation}
\label{app:add_results_inf_eval}

\begin{figure}[tb]
  \begin{center}
    \begin{tabular}{@{}c@{\hskip 1em}c@{\hskip 1em}c@{}}
    \hspace{1.2em}(a) \small\Slide &\hspace{1.2em}(b) \small\FPP & \hspace{1.2em}(c) \small\FPP\\[-0.2em]
    \includegraphics[width=0.30\textwidth]{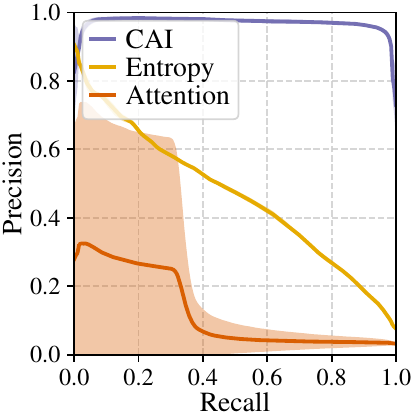}&
    \includegraphics[width=0.30\textwidth]{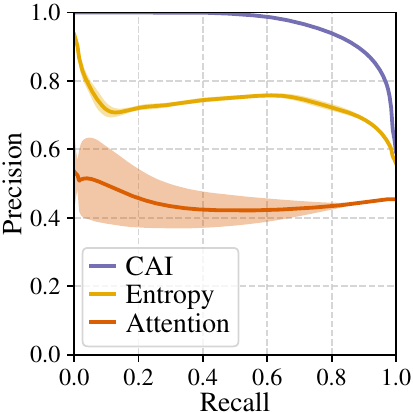}&
    \includegraphics[width=0.30\textwidth]{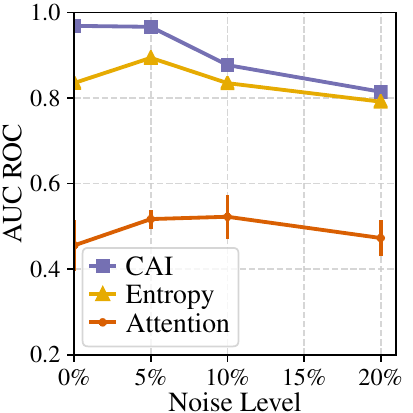}
    \end{tabular}
    \caption{Causal influence detection performance, complementing \fig{fig:inf_eval_roc_curves} in the main part. (a, b) Precision-recall curves on \Slide and \FPP environments. (c) Area-under-ROC curve for \FPP depending on added state noise. Noise level is given as percentage of one standard deviation over the dataset.}
    \label{fig:inf_eval_pr_curves}
    \end{center}
\end{figure}

\begin{figure}[tb]
  \begin{center}
    \begin{tabular}{@{}c@{\hskip 1em}c@{\hskip 1em}c@{}}
    \hspace{1.2em}(a) \small\FPP & \hspace{1.2em}(b) \small\FPP &\hspace{1.2em}\small\FPP \\[-0.2em]
    \includegraphics[width=0.30\textwidth]{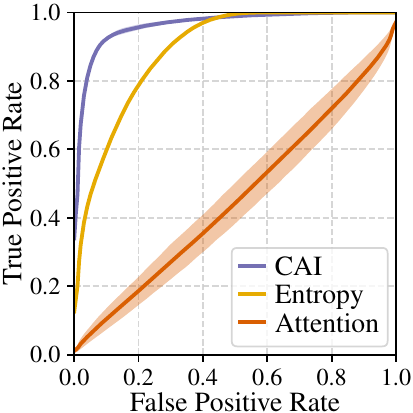} &
    \includegraphics[width=0.30\textwidth]{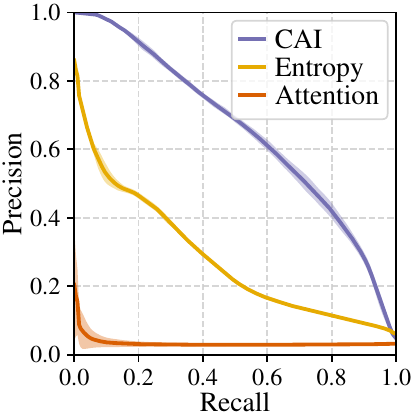} &
    \includegraphics[width=0.30\textwidth]{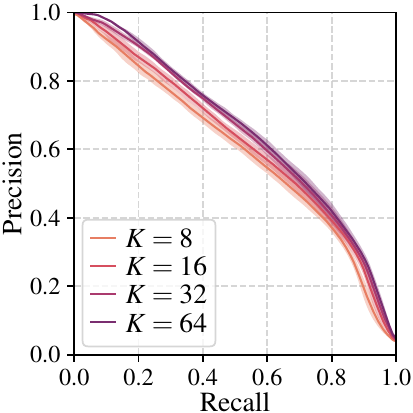}
    \end{tabular}
    \caption{Causal influence detection performance on a test dataset of 250k transitions from a \emph{random policy}. This dataset only has 3.3\% transitions with influence (\ie labeled as positive), which is why the detection task is considerably harder. Left, center: ROC and PR curves. Right: PR curves for CAI when varying the number of actions $K$.}
    \label{fig:inf_eval_curves_rand}
    \end{center}
\end{figure}

In this section, we include additional results for the empirical evaluation of influence detection in \sec{sec:empirical_eval}.
\Fig{fig:inf_eval_pr_curves}a and \fig{fig:inf_eval_curves_rand}b show precision-recall (PR) curves for the experiment in \sec{sec:empirical_eval}, while \fig{fig:inf_eval_pr_curves}c shows how area-under-ROC curve varies while increasing the observation noise level.

For \FPP, we also evaluated on a test dataset obtained from a random policy.
On this dataset, the detection task is considerably harder:
it contains only 3.3\% transitions where the agent has influence (\ie labeled as positive), and it does not contain samples where the agent moves the object in the air, which are easier to detect as influence.
We show ROC and PR curves for this dataset in \fig{fig:inf_eval_curves_rand}a and \fig{fig:inf_eval_curves_rand}b.
As expected, the detection performance drops compared to the other test dataset, which can in particular be seen in the PR curve.
But overall, the performance is still quite good when taking into account the low amount of positive samples.
In \fig{fig:inf_eval_curves_rand}c, we also plot the impact of varying the number of sampled actions $K$ on this dataset.
As one can see, the method is relatively robust to this parameter, giving decent performance even under a small number of sampled actions.
However, we think that a higher number of sampled actions is important in edge-case situations, which are overshadowed in such a quantitative analysis.

Finally, in \fig{fig:pnp_score_analysis}, we give a qualitative analysis of CAI.
Here, we plot the trajectory in three different situations: no contact of agent and object, briefly touching the object, and successfully manipulating the object.
As desired, CAI is low and high in the ``no contact'' and ``successful manipulation'' trajectories.
In the ``brief touch'' trajectory, CAI spikes around the contact.
However, in steps 8-11, when the agent hovers near the object, the heuristical labeling (see \sec{app:env_fetch_pickandplace}) still detects that influence of agent on object is possible, which CAI does not register.
These are difficult cases which show that our method still has room for improvement; we think that these failure cases could be resolved by employing a better model.
However, we also note that the ``ground truth'' label is only based on a heuristic, which will make mispredictions at times as well.

\begin{figure}[htb]
\centering
\hspace{0.5em}
\begin{overpic}[width=0.09\textwidth]{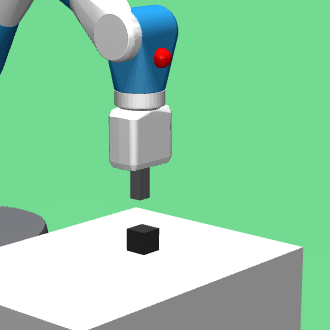} \put (80,78) {\small 0} \end{overpic} \hfill
\begin{overpic}[width=0.09\textwidth]{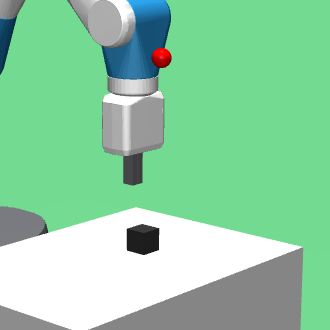} \put (80,78) {\small 2} \end{overpic} \hfill
\begin{overpic}[width=0.09\textwidth]{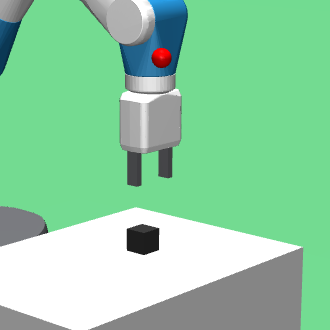} \put (80,78) {\small 4} \end{overpic} \hfill
\begin{overpic}[width=0.09\textwidth]{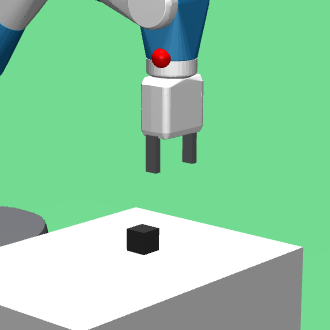} \put (80,78) {\small 6} \end{overpic} \hfill
\begin{overpic}[width=0.09\textwidth]{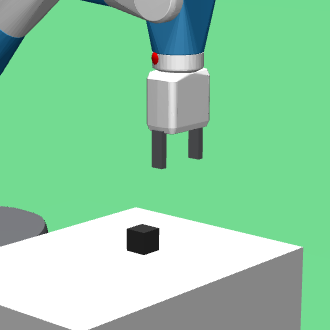} \put (80,78) {\small 8} \end{overpic} \hfill
\begin{overpic}[width=0.09\textwidth]{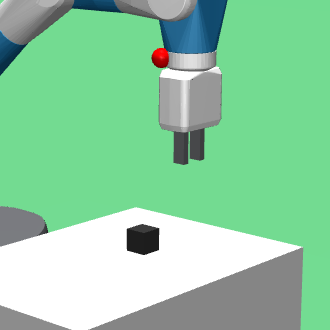} \put (71,78) {\small 10} \end{overpic} \hfill
\begin{overpic}[width=0.09\textwidth]{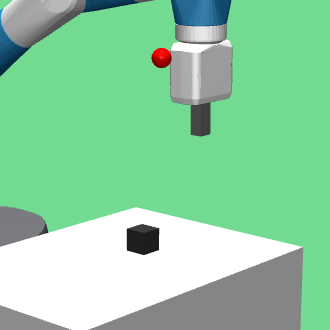} \put (72,78) {\small 12} \end{overpic} \hfill
\begin{overpic}[width=0.09\textwidth]{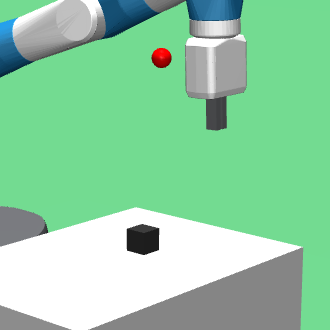} \put (73,78) {\small 14} \end{overpic} \hfill
\begin{overpic}[width=0.09\textwidth]{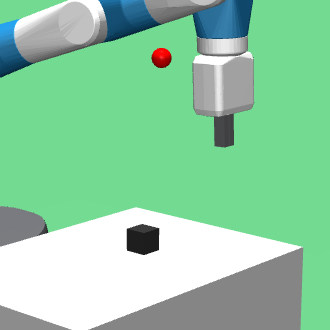} \put (73,78) {\small 16} \end{overpic} \hfill
\begin{overpic}[width=0.09\textwidth]{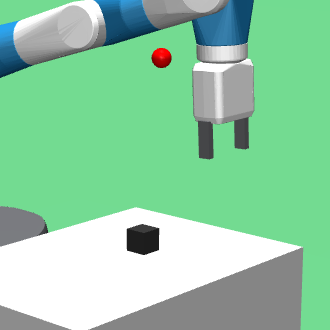} \put (73,78) {\small 18} \end{overpic} \hfill
\\[0.2em]
\includegraphics[width=\textwidth]{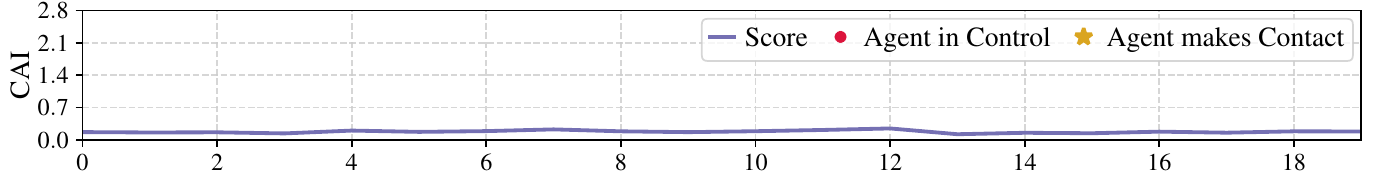}
\\[0.6em]
\hspace{0.5em}
\begin{overpic}[width=0.09\textwidth]{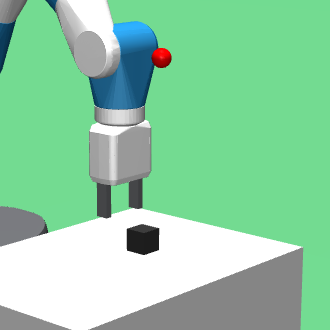} \put (80,78) {\small 0} \end{overpic} \hfill
\begin{overpic}[width=0.09\textwidth]{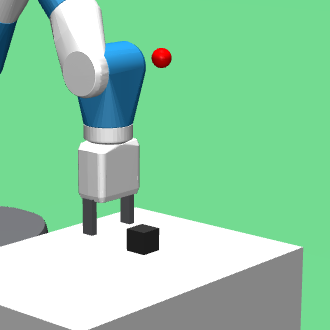} \put (80,78) {\small 2} \end{overpic} \hfill
\begin{overpic}[width=0.09\textwidth]{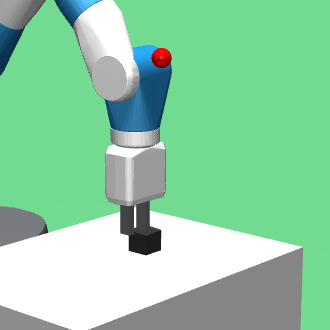} \put (80,78) {\small 4} \end{overpic} \hfill
\begin{overpic}[width=0.09\textwidth]{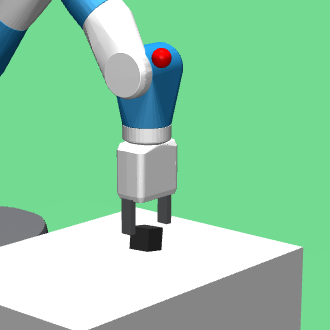} \put (80,78) {\small 6} \end{overpic} \hfill
\begin{overpic}[width=0.09\textwidth]{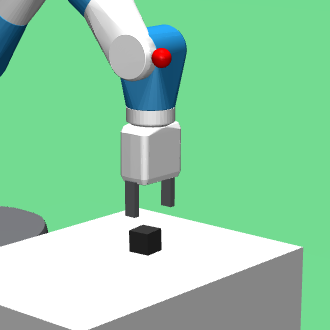} \put (80,78) {\small 8} \end{overpic} \hfill
\begin{overpic}[width=0.09\textwidth]{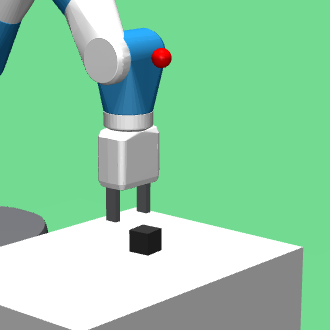} \put (71,78) {\small 10} \end{overpic} \hfill
\begin{overpic}[width=0.09\textwidth]{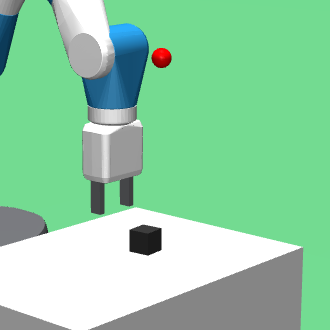} \put (71,78) {\small 12} \end{overpic} \hfill
\begin{overpic}[width=0.09\textwidth]{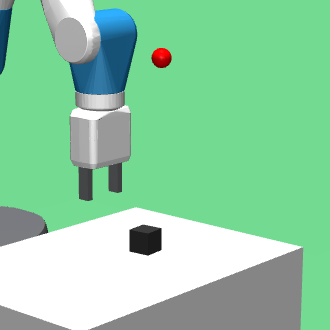} \put (71,78) {\small 14} \end{overpic} \hfill
\begin{overpic}[width=0.09\textwidth]{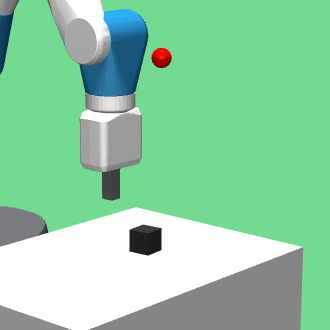} \put (71,78) {\small 16} \end{overpic} \hfill
\begin{overpic}[width=0.09\textwidth]{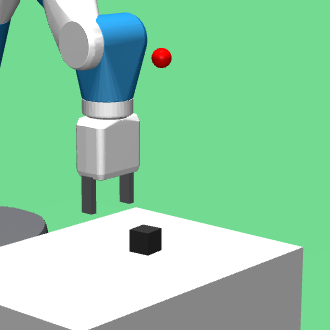} \put (71,78) {\small 18} \end{overpic} \hfill
\\[0.2em]
\includegraphics[width=\textwidth]{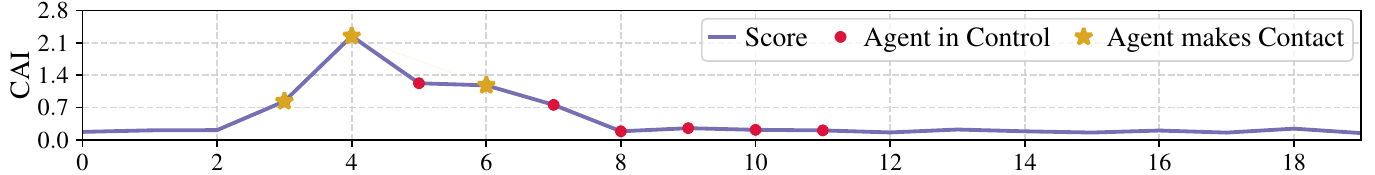}
\\[0.6em]
\hspace{0.5em}
\begin{overpic}[width=0.09\textwidth]{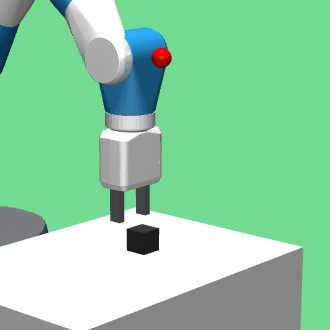} \put (80,78) {\small 0} \end{overpic} \hfill
\begin{overpic}[width=0.09\textwidth]{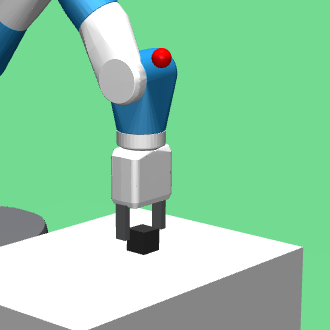} \put (80,78) {\small 2} \end{overpic} \hfill
\begin{overpic}[width=0.09\textwidth]{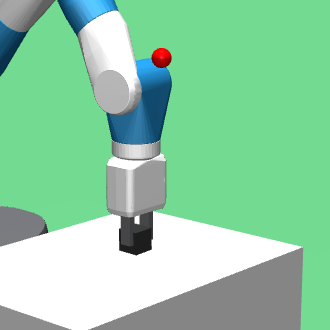} \put (80,78) {\small 4} \end{overpic} \hfill
\begin{overpic}[width=0.09\textwidth]{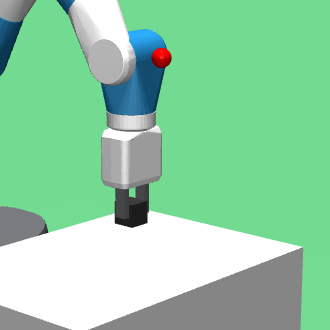} \put (80,78) {\small 6} \end{overpic} \hfill
\begin{overpic}[width=0.09\textwidth]{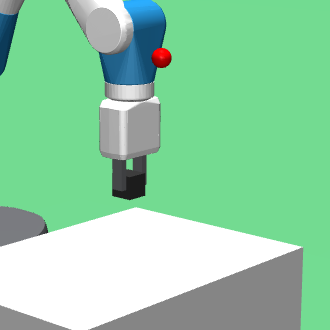} \put (80,78) {\small 8} \end{overpic} \hfill
\begin{overpic}[width=0.09\textwidth]{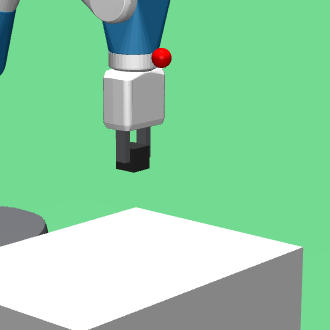} \put (71,78) {\small 10} \end{overpic} \hfill
\begin{overpic}[width=0.09\textwidth]{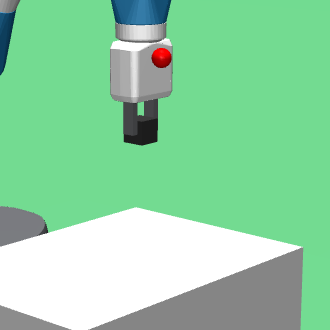} \put (71,78) {\small 12} \end{overpic} \hfill
\begin{overpic}[width=0.09\textwidth]{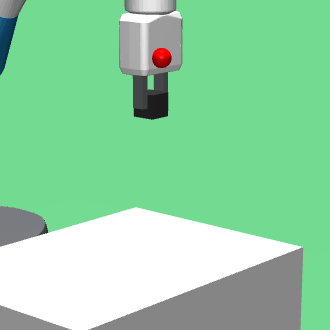} \put (71,78) {\small 14} \end{overpic} \hfill
\begin{overpic}[width=0.09\textwidth]{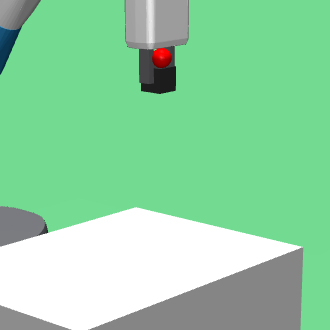} \put (71,78) {\small 16} \end{overpic} \hfill
\begin{overpic}[width=0.09\textwidth]{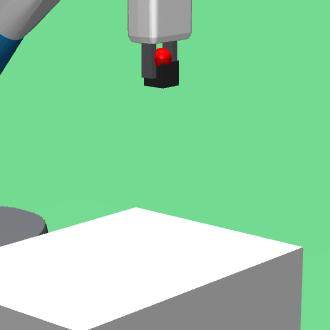} \put (71,78) {\small 18} \end{overpic} \hfill
\includegraphics[width=\textwidth]{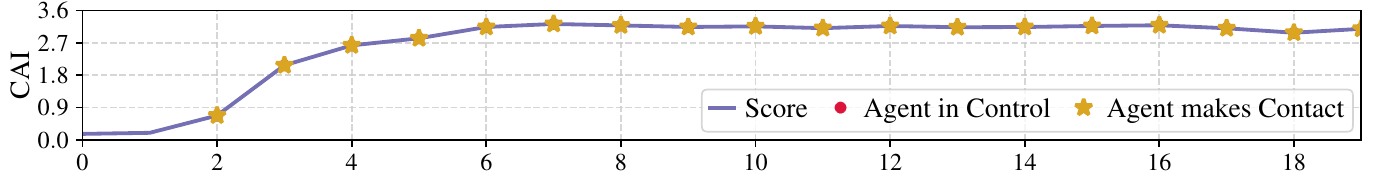}
  \caption{Visualizing CAI score $\CAI(s)$ over first 20 steps of an episode on \FPP. The red dots mark states where the agent has causal influence on the object (according to the ground truth label as described in \sec{app:env_fetch_pickandplace}). Yellow stars mark that the agent makes contact with the object between this state and the next state (as derived from the simulator), which also includes that the agent has causal influence.
  Plotted are episodes with no contact (first row), briefly touching the object (second row), and successful manipulation of the object (third row).}
  \label{fig:pnp_score_analysis}
  \hfill
\end{figure}

\section{Additional Results for Reinforcement Learning}
\label{app:add_results_rl}

\begin{figure}[htb]
    \begin{center}
    \includegraphics[width=0.33\textwidth]{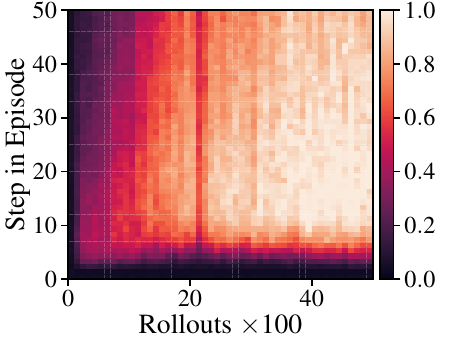}%
    \includegraphics[width=0.33\textwidth]{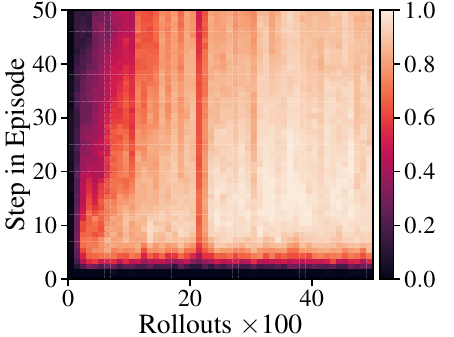}%
    \includegraphics[width=0.33\textwidth]{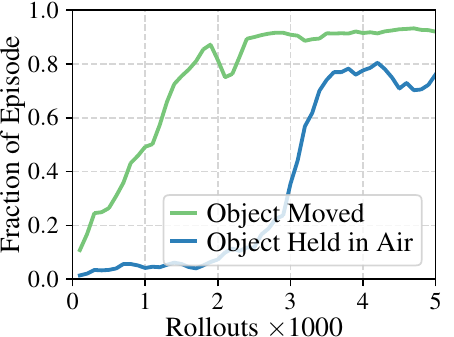}
    \caption{Analyzing the behavior of CAI. Plotted is one run of the intrinsic motivation experiment from \sec{subsec:app_exploration_bonus}.
    Left: heatmap showing score $\CAI(s)$ per step, as stored in replay buffer after \numprint{5000} episodes.
    Score is averaged in groups of 100 episodes and normalized by 95th percentile.
    Center: heatmap showing ground truth label derived for causal influence, as described in \sec{app:env_fetch_pickandplace}.
    CAI approximates the ground truth well.
    Right: behavior of the agent in this run.
    }
    \label{fig:analysis_intrinsic_mot}
    \end{center}
\end{figure}

In this section, we include additional results to the RL experiments in \sec{sec:rl_applications}.
First of, in \fig{fig:analysis_intrinsic_mot}, we analyse CAI's behavior using one the runs from the intrinsic motivation experiment in \sec{subsec:app_exploration_bonus}.
To this end, we plot a heatmap visualising the score distribution in the replay buffer after \numprint{5000} episodes, a corresponding heatmap with ground truth labels, and the fraction of steps where the object was moved or held in the air by the agent.
It can be seen that CAI's distribution approximates the ground truth label's distribution well.
It also becomes visible that CAI measures the \emph{strength of causal influence}, as the score is highest after the agent learns to lift the object in the air (episode \numprint{3000} onwards).
In comparison, the binary ground truth distribution assigns high scores more uniformly.

In \fig{fig:pnp_comparison_all}, we analyse the impact of combining our different proposed improvements, namely prioritization, exploration bonus, and active action selection.
On its own, prioritization brings the largest benefit.
Combining either exploration bonus or active action selection with prioritization leads to similar further improvements.
Both are complementary however; combining all three variants together results in the best performing version.
Compared to not using CAI, combining two or all three improvements leads to a $4$--$10\times$ increase in sample efficiency.

In \fig{fig:analysis_intrinsic_mot}, we plot different versions of the VIME baseline.
For VIME, there is the choice of which parts of the state space the information gain should be computed on.
We compared the variants of using the full state, only the position of the agent and the object, and only the position of the object (which would be similar to CAI).
The variant of using agent and object position performed best, and thus we use it for comparison in \fig{fig:exploration} in the main part.

Finally, in \fig{fig:analysis_ens_dis}, we plot different versions of the ensemble disagreement baseline. 
Similarly to VIME, we also have the option of choosing parts of the state space the disagreement should be computed on. 
We compared the variants of using the full state, only the position of the agent and the object, and only the position of the object (which would be similar to CAI).
The variant of using just object position performed best, and thus we use it for comparison in \fig{fig:exploration} in the main part.

\paragraph{Preliminary Experiments with Negative Results}

In preliminary work, we also experimented with other variants, which we list here for completeness.
For prioritization, we tested using a proportional distribution over episode scores (as in \cite{Schaul2016PER}) instead of the ranked distribution.
While the proportional distribution also worked, it resulted in slower learning and performance did not converge to 100\% success rate.
We also briefly tried shaping the ranking distribution by raising the score by a power before ranking (as in \cite{Schaul2016PER}), but this did not result in notable improvements, so we dropped this line for simplicity.
Further, \citet{Schaul2016PER} use importance sampling to correct for the bias in state sampling introduced by the prioritization.
We found this not to be necessary for the tasks we experimented with, but note that it could be required to converge for other environments.
Last, we also experimented with prioritizing states within an episode, and prioritized selection of goals for HER, but could not achieve any improvements from this.

For the exploration bonus, we tested adding a flat bonus when the score is over a certain threshold, and a bonus that interacts multiplicative with the reward.
Both variants performed worse than the additive bonus.
Linearly annealing the bonus to zero over the course of training did also not result in improvements.
For active action selection, we experimented with sampling actions according to a ranked or proportional distribution over the CAI score instead of taking the action with a maximum score.
Both versions performed worse than maximum score action selection.

\begin{figure}[htb]
  \centering
  \hfill
  \begin{minipage}[t]{.48\textwidth}
    \centering
    \includegraphics[width=0.9375\textwidth]{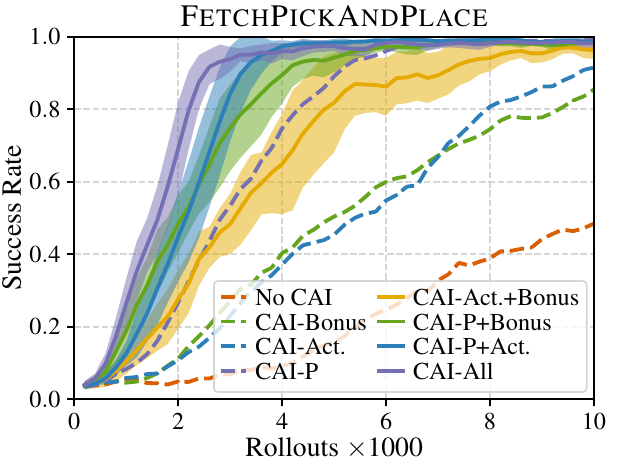}%
    \captionof{figure}{Performance when combining our proposed improvements, namely prioritization (P), exploration bonus (Bonus), active action selection (Act.), in \FPP.
    All proposed improvements act complementary and yield compounded increase in sample efficiency. For clarity, we omit the standard deviation of the single variant versions here, which can be found in the main part.}
    \label{fig:pnp_comparison_all}
  \end{minipage}
  \hfill
  \begin{minipage}[t]{.48\textwidth}
    \centering
    \includegraphics[width=0.9375\textwidth]{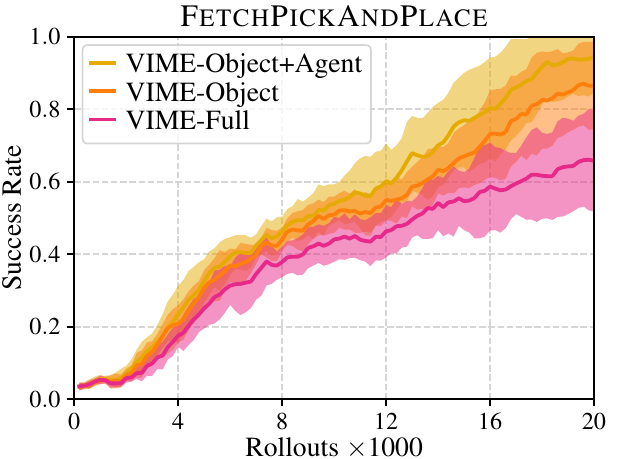}%
    \captionof{figure}{Comparing different variants of the VIME baseline~\cite{Houthooft2016VIME} for exploration in \FPP. The versions differ in the state components used for information gain. VIME-Object: object coordinates. VIME-Object+Agent: object and robotic gripper coordinates. VIME-Full: full state, including velocities and rotation state.}
    \label{fig:vime_baseline}
  \end{minipage}
  \hfill
\end{figure}

\begin{figure}[htb]
    \begin{center}
    \includegraphics[width=0.48\textwidth]{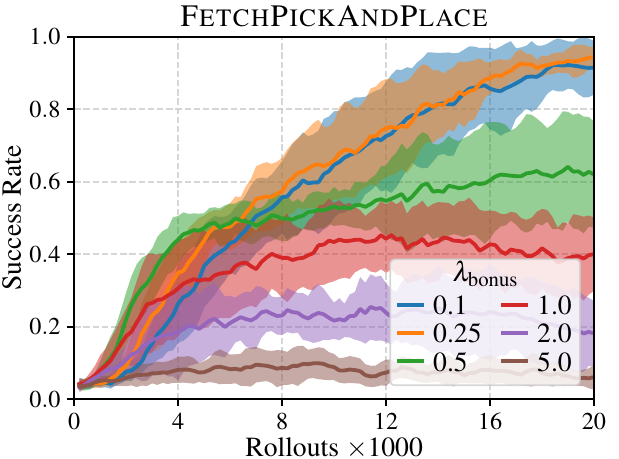}\hfill%
    \includegraphics[width=0.48\textwidth]{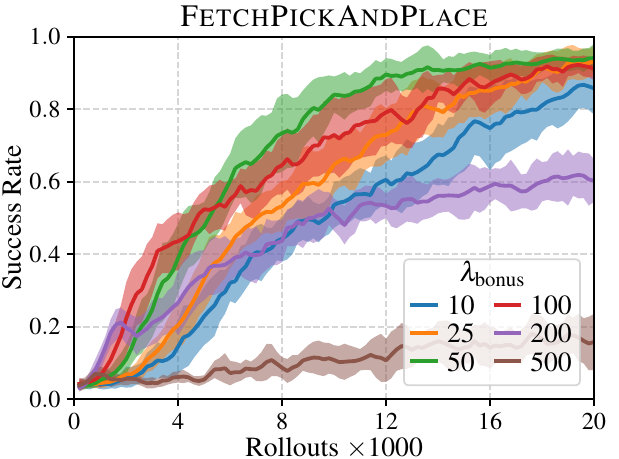}%
    \caption{Comparing different variants of the Ensemble Disagreement baseline~\cite{Pathak2019ExplorationDisagreement}. 
    The versions differ in the state components used for computing the disagreement, and the weight of the exploration bonus $\lambda_{\mathrm{bonus}}$. 
    Left: use full state, including velocities and rotation state, take bonus from time of collecting rollout.
    Right: use only object coordinates, recompute bonus after model updates.
    }
    \label{fig:analysis_ens_dis}
    \end{center}
\end{figure}

\section{Settings of Influence Detection Evaluation}
\label{app:setup_influence_eval}

\begin{center}
 Code is available under \url{https://github.com/martius-lab/cid-in-rl}.
\end{center}

\paragraph{Datasets}

For the experiment in \sec{sec:empirical_eval}, we use separate training, validation and test datasets.
For \Slide, training and validation set consist of \numprint{1000} episodes (\numprint{30000} samples) collected from two separate training runs of a DDPG+HER agent.
The test set consists of \numprint{4000} episodes (\numprint{120000} samples), where \numprint{2000} episodes were collected by a random policy, and \numprint{2000} by a DDPG+HER agent.
For \FPP, training and validation set consist of \numprint{5000} episodes (\numprint{250000} samples) collected from two separate training runs of a DDPG+HER agent.
The test set consists of \numprint{7500} episodes (\numprint{375000} samples) collected from training a DDPG+HER agent for \numprint{30000} epochs.
The test set was subsampled by selecting \numprint{250} random episodes from every \numprint{1000} collected episodes (subsampling was performed to ease the computational load of evaluation).
The test set has 45.3\% samples that are labeled positive, \ie the agent has influence.
We note that it is not strictly necessary to train on data collected from a training RL agent, as we also trained on data from random policies and obtained similar results.
In \sec{sec:empirical_eval}, we also analyse the behavior when adding observation noise, which is given as a percentage level.
To determine an appropriate level of noise, we recorded the standard deviations over the different state dimensions on a dataset with \numprint{10000} episodes collected from training a DDPG+HER agent.
We then add Gaussian noise to each dimension of the state with standard deviation equal to a percentage of the dataset standard deviation.

In \fig{fig:inf_eval_curves_rand}, we also give results for \FPP on a different test set.
This test set consists of \numprint{5000} episodes (\numprint{250000} samples) collected from a random policy.
It only has 3.3\% samples that are labeled positive, \ie there is considerably less interaction with the object than on the other test set.

\paragraph{Methods}

The classification scores for the different methods were obtained as follows. For CAI, we use $\hat{\CAI}(s)$ as given in \eqn{eq:cmi_approx}, with $K=64$ actions for sampling. For entropy, we use the same model as trained for CAI, and estimate conditional entropy as $H(S'_j \mid S=s) \approx \frac{1}{K} \sum_{i=1}^K H \g[ p(s_j' \mid s, a^{(i)}) \g]$ with $\{a^{(1)}, \ldots, a^{(K)}\} \stackrel{\text{iid}}{\sim} \pi$, using the formula given in \eqn{eqn:entropy_gaussian} for the entropy of a Gaussian.
For attention, we use the attention weights of a Transformer, where the score is computed as follows.
As the Transformer requires a set as input, we split the input vector into components for the agent state, the object state, and the action.
The Transformer is trained to predict the position of the object, that is, we discard the Transformer's output for the other components.
Then, letting $A_i$ denote the attention matrix of the $i$'th of $N$ layers, the total attention matrix is computed as $\prod_{i=1}^N A_i$.
The score is the entry of the total attention matrix corresponding to the input position of the action component, and the output position of the object component.
We refer to \citet{Pitis2020CFDataAug} for more details.

We list the model hyperparameters for CAI in \tab{table:eval_inf_settings_cai} and for the Transformer in \tab{table:eval_inf_settings_transformer}.
Training was performed for a maximum of \numprint{3000} epochs on \Slide and \numprint{2000} epochs on \FPP.
The training was stopped early when the mean-squared-error (MSE) did not improve for 10 evaluations on the validation set, where evaluations were conducted every 20 training epochs.
We trained all models using the Adam optimizer~\cite{Kingma2015Adam}, with $\beta_1=0.9$, $\beta_2=0.999$ if not noted otherwise.
All models were trained to predict the relative difference to the next state instead of the absolute next state, \ie the target was $S'_j-S_j$.
For \FPP, we rescaled the targets by a factor of $50$, as it resulted in better performance.

We used a simple multi-layer perceptron (MLP) for the model in CAI, with two separate output layers for mean and variance.
To constrain the variance to positive range, the variance output of the MLP was processed by a softplus function (given by $\log(1 + \exp(x))$), and a small positive constant of $10^{-8}$ was added to prevent instabilities near zero.
We also clip the variance to a maximum value of $200$.
For weight initialization, orthogonal initialization~\cite{Saxe2014OrthogonalInit} was used.
We observed that training with the Gaussian likelihood loss (\eqn{eq:log_likelihood}) can be unstable.
Applying spectral normalization~\cite{Miyato2018SpectralNorm} to some of the layers decreased the instability considerably.
Note that we did not apply spectral normalization to the mean output layer, as doing so resulted in worse predictive performance.
For \FPP, the inputs were normalized by applying batch normalization (with no learnable parameters) before the first layer.

\begin{table*}[htb!]
    \centering
    \caption{Settings for CAI on Influence Evaluation.}
    \label{table:eval_inf_settings_cai}
    \begin{subtable}[t]{.5\textwidth}
        \centering
        \caption{\Slide settings.}
        \begin{tabular}{@{}ll@{}}
        \toprule
        \textbf{Parameter} & \textbf{Value} \\
        \midrule
        Batch Size & $1000$ \\
        Learning Rate & $0.0003$ \\
        Network Size & $4 \times 128$ \\
        Activation Function & ReLU \\
        Spectral Norm on $\sigma$ & Yes \\
        Spectral Norm on Layers & Yes \\
        Normalize Input & No \\
        \bottomrule
        \end{tabular}
    \end{subtable}\hfill%
    \begin{subtable}[t]{.5\textwidth}
        \centering
        \caption{\FPP settings.}
        \begin{tabular}{@{}ll@{}}
        \toprule
        \textbf{Parameter} & \textbf{Value} \\
        \midrule
        Batch Size & $500$ \\
        Learning Rate & $0.0008$ \\
        Network Size & $3 \times 256$ \\
        Activation Function & ReLU \\
        Spectral Norm on $\sigma$ & Yes \\
        Spectral Norm on Layers & Yes \\
        Normalize Input & Yes \\
        \bottomrule
        \end{tabular}
    \end{subtable}
\end{table*}

\begin{table*}[htb!]
    \centering
    \caption{Settings for Transformer on Influence Evaluation.}
    \label{table:eval_inf_settings_transformer}
    \begin{subtable}[t]{.5\textwidth}
        \centering
        \caption{\Slide settings.}
        \begin{tabular}{@{}ll@{}}
        \toprule
        \textbf{Parameter} & \textbf{Value} \\
        \midrule
        Batch Size & $1000$ \\
        Learning Rate & $0.0003$ \\
        Embedding Dimension & $16$ \\
        FC Dimension & $32$ \\
        Number Attention Heads & $1$ \\
        Number Transformer Layers & $2$ \\
        Normalize Input & No \\
        \bottomrule
        \end{tabular}
    \end{subtable}\hfill%
    \begin{subtable}[t]{.5\textwidth}
        \centering
        \caption{\FPP settings.}
        \begin{tabular}{@{}ll@{}}
        \toprule
        \textbf{Parameter} & \textbf{Value} \\
        \midrule
        Batch Size & $500$ \\
        Learning Rate & $0.0001$ \\
        Embedding Dimension & $128$ \\
        FC Dimension & $128$ \\
        Number Attention Heads & $2$ \\
        Number Transformer Layers & $3$ \\
        Normalize Input & Yes \\
        \bottomrule
        \end{tabular}
    \end{subtable}
\end{table*}

\section{Settings of Reinforcement Learning Experiments}
\label{app:hyperparameters_rl}

Our RL experiments are run in the goal-conditioned setting, that is, each episode, a goal is created from the environment that the agent has to reach.
An episode counts as success when the goal is reached upon the last step of the episode.
On the \textsc{Fetch} environments we use, the goal are coordinates where the object has to be moved to.
For goal sampling, we use the settings as given by the environments in OpenAI Gym (described in more detail by~\citet{Plappert2018MultiGoalRL}).
We use the sparse reward setting, that is, the agent receives a reward of $0$ when the goal is reached, and $-1$ otherwise.
Practically, goal-conditioned RL is implemented by adding the current goal to the input for policy and value function.
For evaluating the RL experiments, we run the current RL agent $100$-times every $200$ episodes, and average the outcomes to obtain the success rate.

For the RL experiments, we use the same base settings for all algorithms, listed in \tab{table:rl_ddpg_hyperparams}.
These settings (for DDPG+HER) were borrowed from~\citet{Ren2019Exploration}, as they already provide excellent performance compared to the original settings from~\citet{Plappert2018MultiGoalRL} (\eg on \FPP, $4\times$ faster to reach 90\% success rate).
Before starting to train the agent, the memory buffer is seeded with $200$ episodes collected from a random policy.
The input for policy and value function is normalized by tracking mean and standard deviation of inputs seen over the course of training.

All experiments can be run on a standard multi-core CPU machine.
Training a CAI model with prioritization for \numprint{20000} epochs on \FPP takes roughly 8 hours using 3 cores on an Intel Xeon Gold 6154 CPU with 3 GHz.
Training a DDPG+HER agent takes roughly 4 hours.

\paragraph{CAI Implementation}
For CAI, we list the settings in \tab{table:rl_cai_hyperparams}.
After an episode is collected by the agent, the CAI score $\CAI$ is computed for the episode and stored in the replay buffer.
The model for CAI was trained every $100$ episodes by sampling batches from the replay buffer.
For online training, we designed a training schedule where the number of training batches used is varied over the course of training.
We note that the specific schedule used appeared to be not that important as long as the model is sufficiently trained initially.
We mostly chose our training schedule to make training computationally cheaper.
After every round of model training, the CAI score $\CAI$ is recomputed on the states stored in the replay buffer.
With larger buffer sizes, this can become time consuming.
We observed that it is also possible to fully recompute scores only every \numprint{1000} epochs, and otherwise recompute scores only on the newest \numprint{1000} epochs, at only a small loss of performance.

\paragraph{Exploration Bonus}
For the exploration bonus experiments, we clip the scores at specific values (reported in the settings tables under ``Maximum Bonus'') to restrict the influence of outliers before multiplying with $\lambda_\text{\text{bonus}}$ and addition to the task reward.
Moreover, the total reward is clipped to a maximum value of $0$ to prevent making states more attractive than reached-goal states.

\paragraph{Baselines}
We list specific hyperparameters for the baselines VIME~\cite{Houthooft2016VIME}, ensemble disagreement~\cite{Pathak2019ExplorationDisagreement}, and PER~\cite{Schaul2016PER} in \tab{table:rl_baselines_hyperparams}.
For VIME~\cite{Houthooft2016VIME}, we adapted a Pytorch port\footnote{Alec Tschantz. \url{https://github.com/alec-tschantz/vime}. MIT license.} of the official implementation\footnote{OpenAI. \url{https://github.com/openai/vime}. MIT license.}.
For PER, we used the hyperparameters as proposed in~\citet{Schaul2016PER} and implemented it ourselves.
For EBP~\cite{Zhao2018EBP}, we adapted the official implementation\footnote{Rui Zhao. \url{https://github.com/ruizhaogit/EnergyBasedPrioritization}. MIT license.}.

Ensemble disagreement~\cite{Pathak2019ExplorationDisagreement} was implemented by ourselves. 
We experimented with two variants.
The first variant (\emph{Full State + Direct Bonus} in \tab{table:rl_baselines_hyperparams}) is close to the original, that is, we use the full state prediction for the bonus computation, train the model every $10$ collected rollouts and compute the bonus for each state using the current model when collecting that state.
The second variant (\emph{Object Position + Recompute Bonus} in \tab{table:rl_baselines_hyperparams}) is closer to CAI, that is, we only use the position of the object for bonus computation, train the model every $100$ collected rollouts, and fully recompute the bonus for each collected state every \numprint{1000} epochs (see CAI implementation above).
For both variants, we use the same neural network architectures as CAI and train the models using the mean squared error loss.

\begin{table*}[htb]
    \centering
    \caption{Base settings for DDPG+HER. Also used for CAI, VIME, EBP, and PER.}
    \label{table:rl_ddpg_hyperparams}
    \begin{subtable}[t]{.5\textwidth}
        \centering
        \caption{General settings.}
        \begin{tabular}{@{}ll@{}}
        \toprule
        \textbf{Parameter} & \textbf{Value} \\
        \midrule
        Episode Length & $50$ \\
        Batch Size & $256$ \\
        Updates per Episode & $20$ \\
        Replay Buffer Warmup & $200$ \\
        Replay Buffer Size & \numprint{500000} \\
        Learning Rate & $0.001$ \\
        Discount Factor $\gamma$ & $0.98$ \\
        Polyak Averaging & $0.95$ \\
        Action $L_2$ Penalty & $1$ \\
        Action Noise & $0.2$ \\
        Random $\epsilon$-Exploration & $0.3$ \\
        Observation Clipping & $[-5, 5]$ \\
        Q-Target Clipping & $[-50, 0]$ \\
        Policy Network & $3 \times 256$ \\
        Q-Function Network & $3 \times 256$ \\
        Activation Function & ReLU \\
        Weight Initialization & Xavier Uniform~\cite{Glorot2010XavierInit} \\
        Normalize Input & Yes \\
        HER Replay Strategy & Future \\
        HER Replay Probability & $0.8$ \\
        \bottomrule
        \end{tabular}
    \end{subtable}%
    \begin{subtable}[t]{.5\textwidth}
        \centering
        \caption{Environment/task-specific settings.}
        \begin{tabular}{@{}ll@{}}
        \toprule
        \multicolumn{2}{c}{\FRT} \\
        \textbf{Parameter} & \textbf{Value} \\
        \midrule
        Learning Rate & $0.0003$ \\
        Discount Factor $\gamma$ & $0.95$ \\
        Polyak Averaging & $0.99$ \\
        \bottomrule
        \\
        \toprule
        \multicolumn{2}{c}{\FPP} \\
        \multicolumn{2}{c}{\emph{Intrinsic Motivation}} \\
        \textbf{Parameter} & \textbf{Value} \\
        \midrule
        Replay Buffer Warmup & $100$ \\
        Learning Rate & $0.003$ \\
        Discount Factor $\gamma$ & $0.95$ \\
        Polyak Averaging & $0.99$ \\
        Q-Function Network & $2 \times 192$ \\
        Q-Target Clipping & None \\
        HER Replay Strategy & No HER \\
        \bottomrule
        \end{tabular}
    \end{subtable}
\end{table*}

\begin{table*}[htb]
    \centering
    \caption{Settings for CAI in RL experiments.}
    \label{table:rl_cai_hyperparams}
    \begin{subtable}[t]{.5\textwidth}
        \centering
        \caption{General settings.}
        \begin{tabular}{@{}ll@{}}
        \toprule
        \textbf{Parameter} & \textbf{Value} \\
        \midrule
        Batch Size & $500$ \\
        Train Model Every & $100$ Episodes \\
        Training Schedule & \\
        \quad Initial ($200$ Episodes) & \numprint{40000} Batches \\
        \quad $\leq \numprint{5200}$ Episodes & \numprint{10000} Batches \\
        \quad $\leq \numprint{10200}$ Episodes & \numprint{5000} Batches \\
        \quad $> \numprint{10200}$ Episodes & No Training \\
        Adam $\beta_2$ & $0.9$ \\
        Learning Rate & $0.0008$ \\
        Network Size & $4 \times 256$ \\
        Activation Function & Tanh \\
        Spectral Norm on $\sigma$ & Yes \\
        Normalize Input & Yes \\
        CAI Number of Actions $K$ & 32 \\
        \bottomrule
        \end{tabular}
    \end{subtable}\hfill%
    \begin{subtable}[t]{.5\textwidth}
        \centering
        \caption{Environment/task-specific settings.}
        \begin{tabular}{@{}ll@{}}
        \toprule
        \multicolumn{2}{c}{\FRT} \\
        \textbf{Parameter} & \textbf{Value} \\
        \midrule
        Network Size & $4 \times 386$ \\
        \bottomrule
        \\[-0.6em]
        \toprule
        \multicolumn{2}{c}{\FPP} \\
        \multicolumn{2}{c}{\emph{Intrinsic Motivation}} \\
        \textbf{Parameter} & \textbf{Value} \\
        \midrule
        Training Schedule & \\
        \quad $\leq \numprint{5200}$ Episodes & \numprint{20000} Batches \\
        Maximum Bonus & 10 \\
        \bottomrule
        \\[-0.6em]
        \toprule
        \multicolumn{2}{c}{\emph{Exploration Bonus}} \\
        \textbf{Parameter} & \textbf{Value} \\
        \midrule
        Maximum Bonus & 2 \\
        \bottomrule
        \end{tabular}
    \end{subtable}
\end{table*}

\begin{table*}[htb]
    \centering
    \caption{Settings for Baselines in RL experiments.}
    \label{table:rl_baselines_hyperparams}
    \begin{subtable}[t]{.5\textwidth}
        \centering
        \vskip 0pt  %
        \begin{tabular}{@{}ll@{}}
        \toprule
        \multicolumn{2}{c}{\emph{VIME}~\cite{Houthooft2016VIME}} \\
        \textbf{Parameter} & \textbf{Value} \\
        \midrule
        Batch Size & $25$ \\
        Train Model Every & $1$ Episode \\
        Batches per Update & $100$ \\
        Learning Rate & $0.003$ \\
        Bayesian NN Size & $2 \times 64$ \\
        Number of ELBO Samples & 10 \\
        KL Term Weight  & $0.05$ \\
        Likelihood $\sigma$ & 0.5 \\
        Prior $\sigma$ & 0.1 \\
        Normalize Input & Yes \\
        Batch Size Info Gain & 10 \\
        Update Steps Info Gain & 10 \\
        Bonus Weight & 0.3 \\
        Maximum Bonus & 10 \\
        \bottomrule
        \end{tabular}
    \end{subtable}\hfill%
    \begin{subtable}[t]{.5\textwidth}
        \centering
        \vskip 0pt  %
        \begin{tabular}{@{}ll@{}}
        \toprule
        \multicolumn{2}{c}{\emph{Ensemble Disagreement}~\cite{Pathak2019ExplorationDisagreement}} \\
        \textbf{Parameter} & \textbf{Value} \\
        \midrule
        Number of Ensembles & $5$ \\
        Batch Size & $500$ \\
        Adam $\beta_2$ & $0.9$ \\
        Learning Rate & $0.0008$ \\
        Network Size & $4 \times 256$ \\
        Activation Function & Tanh \\
        Normalize Input & Yes \\
        \midrule
        \multicolumn{2}{l}{\emph{Full State + Direct Bonus}} \\
        Train Model Every & $10$ Episodes \\
        Train For & \numprint{500} Batches \\
        Maximum Bonus & 5 \\
        \midrule
        \multicolumn{2}{l}{\emph{Object Position + Recompute Bonus}} \\
        Train Model Every & $100$ Episodes \\
        Train For & \numprint{5000} Batches \\
        Maximum Bonus & 0.1 \\
        \bottomrule
        \\[-0.6em]
        \toprule
        \multicolumn{2}{c}{\emph{Prioritized Experience Replay}~\cite{Schaul2016PER}} \\
        \textbf{Parameter} & \textbf{Value} \\
        \midrule
        $\alpha$ & $0.6$ \\
        $\beta$ & Linearly Scheduled from $0.4$ \\
         &  to $1.0$ over \numprint{20000} Epochs \\
        \bottomrule
        \end{tabular}
    \end{subtable}
\end{table*}

\end{document}

%% file: al_header.tex
\usepackage{graphics}
\usepackage{mathtools}

\usepackage{algorithm}
\usepackage{algorithmic}
\usepackage{enumitem}

\usepackage{amsfonts}       %
\usepackage{amsmath}       %
\usepackage{amssymb}
\usepackage{amsthm}

\newcommand{\Fig}[1]{Figure~\ref{#1}}  %
\newcommand{\fig}[1]{Fig.~\ref{#1}}    %
\newcommand{\tab}[1]{Table~\ref{#1}}

\newcommand{\eqn}[1]{Eq.~\ref{#1}} %
\newcommand{\eqnp}[1]{(Eq.~\ref{#1})} %
\renewcommand{\sec}[1]{Sec.~\ref{#1}} %
\newcommand{\supp}[1]{Suppl.~\ref{#1}}

\usepackage{xspace}
\makeatletter
\DeclareRobustCommand\onedot{\futurelet\@let@token\@onedot}
\def\@onedot{\ifx\@let@token.\else.\null\fi\xspace}
\def\eg{e.g\onedot}
\def\ie{i.e\onedot}

\makeatother

\newcommand{\Real}{\ensuremath{\mathbb R}}        %
\DeclareMathOperator*{\argmin}{arg\,min}
\DeclareMathOperator*{\argmax}{arg\,max}
\newcommand{\Exp}{\ensuremath{\mathbb{E}}}        %
\newcommand{\dx}[1]{\ensuremath{\mathrm{d}#1}}                   %

\let\originalleft\left
\let\originalright\right
\renewcommand{\left}{\mathopen{}\mathclose\bgroup\originalleft}
\renewcommand{\right}{\aftergroup\egroup\originalright}

\newcommand{\g}[1]{%
  \ifthenelse{\equal{#1}{(}}
  {\left( }%
    { \ifthenelse{\equal{#1}{)}}
      { \right)}%
    { \ifthenelse{\equal{#1}{[}}
      {\left[}%
        { \ifthenelse{\equal{#1}{]}}
          { \right]}%
        {#1}}
    }
  }
}

\newcommand{\KL}[2]{\text{D}_{\text{KL}}\g(\left.#1\;\right\|\;#2\g)}
\newcommand{\KLapp}[1]{\text{D}_{\text{#1}}}

\newcommand{\indep}{\mathrel{\text{\raisebox{-0.05em}{\scalebox{1.07}{$\perp\mkern-11mu\perp$}}}}}
\newcommand{\nindep}{\mathrel{\text{\raisebox{-0.05em}{\scalebox{1.07}{$\not\mkern-2mu\perp\mkern-11mu\perp$}}}}}

\usepackage{xcolor}
\definecolor{ourblue}{rgb}{0.368,0.507,0.71}
\definecolor{ourorange}{rgb}{0.881,0.611,0.142}
\definecolor{ourgreen}{rgb}{0.56,0.692,0.195}
\definecolor{ourred}{rgb}{0.923,0.386,0.209}
\definecolor{ourviolet}{rgb}{0.528,0.471,0.701}
\definecolor{ourbrown}{rgb}{0.772,0.432,0.102}
\definecolor{ourlightblue}{rgb}{0.364,0.619,0.782}
\definecolor{ourdarkgreen}{rgb}{0.572,0.586,0.}
\definecolor{ourdarkblue}{RGB}{28,99,148}
\definecolor{ourdarkred}{RGB}{169,53,17}

\usepackage{etoolbox}
\makeatletter
\patchcmd{\NAT@test}{\else \NAT@nm}{\else \NAT@nmfmt{\NAT@nm}}{}{}
\DeclareRobustCommand\citeposs
  {\begingroup
   \let\NAT@nmfmt\NAT@posfmt%
   \NAT@swafalse\let\NAT@ctype\z@\NAT@partrue
   \@ifstar{\NAT@fulltrue\NAT@citetp}{\NAT@fullfalse\NAT@citetp}}
   
\let\NAT@orig@nmfmt\NAT@nmfmt
\def\NAT@posfmt#1{\NAT@orig@nmfmt{#1's}}
\makeatother